\documentclass[10pt,twocolumn,letterpaper]{article}

\usepackage{cvpr}              

\usepackage[utf8]{inputenc} 
\usepackage[T1]{fontenc}    
\usepackage{hyperref}       
\usepackage{url}            
\usepackage{booktabs}       
\usepackage{amsfonts}       
\usepackage{nicefrac}       
\usepackage{microtype}      
\usepackage{xcolor}         

\usepackage{helvet} 
\usepackage{courier}  
\usepackage[utf8]{inputenc} 
\usepackage[T1]{fontenc}    
\usepackage{hyperref}
\usepackage{url}

\usepackage{booktabs}       
\usepackage{amsfonts}       
\usepackage{nicefrac}       
\usepackage{microtype}      

\usepackage{epsfig}
\usepackage{amsmath}
\usepackage{amssymb}
\usepackage{adjustbox}

\usepackage{amsthm}  
\usepackage{wasysym}
\usepackage[ruled,vlined,algo2e]{algorithm2e}
\providecommand{\SetAlgoLined}{\SetLine}
\usepackage{color}
\usepackage{dsfont}	
\usepackage{wrapfig}
\usepackage{footnote}
\usepackage{epstopdf}
 \usepackage{enumitem}
\usepackage{caption}
\usepackage{comment}
\usepackage{multirow}
\usepackage{algorithm}
\usepackage{algorithmic}
\usepackage{mathtools}
\usepackage{subcaption}
\usepackage{stackrel}

\def\eg{\emph{e.g., }}
\def\ie{\emph{i.e., }}

\def\vs{\emph{vs. }}

\def\etal{\emph{et al. }}

\usepackage{pifont}

\usepackage{mathtools}

\makeatletter
\newcommand*{\rom}[1]{\expandafter\@slowromancap\romannumeral #1@}
\makeatother
\makeatletter
\newcommand\footnoteref[1]{\protected@xdef\@thefnmark{\ref{#1}}\@footnotemark}
\makeatother
\newcommand{\bfsection}[1]{\vspace*{0.1cm}\noindent\textbf{#1.}}

\newtheorem{lemma}{Lemma}
\newtheorem{defi}{Definition}

\usepackage[capitalize]{cleveref}
\crefname{section}{Sec.}{Secs.}
\Crefname{section}{Section}{Sections}
\Crefname{table}{Table}{Tables}
\crefname{table}{Tab.}{Tabs.}

\def\cvprPaperID{11310} 
\def\confName{CVPR}
\def\confYear{2023}

\begin{document}

\title{PRISE: Demystifying Deep Lucas-Kanade with Strongly Star-Convex Constraints for Multimodel Image Alignment}

\author{Yiqing Zhang \hspace{1cm} Xinming Huang \hspace{1cm} Ziming Zhang \\
Worcester Polytechnic Institute\\
100 Institute Rd, Worcester, MA, USA\\
{\tt\small \{yzhang37, xhuang, zzhang15\}@wpi.edu}
}

\maketitle

\begin{abstract}
    The Lucas-Kanade (LK) method is a classic iterative homography estimation algorithm for image alignment, but often suffers from poor local optimality especially when image pairs have large distortions. To address this challenge, in this paper we propose a novel {\em Dee\underline{p} Sta\underline{r}-Convexif\underline{i}ed Luca\underline{s}-Kanad\underline{e} (PRISE)} method for multimodel image alignment by introducing strongly star-convex constraints into the optimization problem. Our basic idea is to enforce the neural network to approximately learn a star-convex loss landscape around the ground truth give any data to facilitate the convergence of the LK method to the ground truth through the high dimensional space defined by the network. This leads to a minimax learning problem, with contrastive (hinge) losses due to the definition of strong star-convexity that are appended to the original loss for training. We also provide an efficient sampling based algorithm to leverage the training cost, as well as some analysis on the quality of the solutions from PRISE. We further evaluate our approach on benchmark datasets such as MSCOCO, GoogleEarth, and GoogleMap, and demonstrate state-of-the-art results, especially for small pixel errors. Code can be downloaded from \url{https://github.com/Zhang-VISLab}.
\end{abstract}

\section{Introduction}
\label{sec:intro}
Deep learning networks have achieved great success in homography estimation by directly predicting the transformation matrix in various scenarios. However, the existing classic algorithms still take the place for showing more explainability compared with the deep learning architectures. Such algorithms are often rooted from well-studied theoretical and empirical grounding. Current works often focus on combining the robustness of deep learning with explainability of classical algorithms to handle multimodel image alignment such as image modality and satellite modality. However, due to the high nonconvexity in homography estimation, such methods often suffer from poor local optimality.

Recently Zhao \etal \cite{zhao2021deep} proposed DeepLK for multimodel image alignment, \ie estimating the homography between two planar projections of the same view but across different modalities such as map and satellite images (see Sec. \ref{sssec:homography} for formal definition), based on the LK method \cite{lucasiterative}. This method consists of two novel components:
\begin{itemize}[nosep, leftmargin=*]
    \item A new deep neural network was proposed to map images from different modalities into the same feature space where the LK method can align them.

    \item A new training algorithm was proposed as well by enforcing the local change on the loss landscape should be no less than a quadratic shape centered at the ground truth for any image pair, with no specific reason.
\end{itemize}
Surprisingly, when we evaluate DeepLK based on the public code\footnote{\url{https://github.com/placeforyiming/CVPR21-Deep-Lucas-Kanade-Homography}}, the proposed network cannot work well without the proposed training algorithm. This strongly motivate us to discover {\em the mysteries in the DeepLK training algorithm.}

\bfsection{Deep Reparametrization}
Our first insight from DeepLK is that the deep neural network essentially maps the alignment problem into a much higher dimensional space by introducing a large amount of parameters. The high dimensional space provides the feasibility to reshape the loss landscape of the LK method. Such deep reparametrization has been used as a means of reformulating some problems such as shape analysis \cite{celledoni2022deep}, super-resolution and denoising \cite{bhat2021deep}, while preserving the properties and constraints in the original problems. This insight at test time can be interpreted as
\begin{align}\label{eqn:downstream_task}
    \min_{\omega\in\Omega} \ell(\omega; x) \xRightarrow[\text{via deep learning}]{\text{reparametrization}} \min_{\omega\in\Omega} \ell_f(\omega; x, \theta^*),
\end{align}
where $x\in\mathcal{X}$ denotes the input data, $\ell$ denotes a nonconvex differentiable function (\eg the LK loss) parametrized by $\omega\in\Omega$, $f:\mathcal{X}\times\Theta\rightarrow\mathcal{X}$ denotes an auxiliary function presented by a neural network with learned weights $\theta^*\in\Theta$ (\eg the proposed network in DeepLK), and $\ell_f$ denotes the loss with deep reparametrization (\eg the DeepLK loss). In this way, the learning problem is how to train the network so that the optimal solutions can be located using gradient descent (GD) given data.

\bfsection{Convex-like Loss Landscapes}
Our second insight from DeepLK is that the learned loss landscape from their training algorithm tends to be convex-like (see their experimental results). This is an interesting observation, as it is evidenced in  \cite{li2018visualizing} that empirical more convex-like loss landscapes often return better performance. However, we cannot find any explicit explanation through the paper about the reason, which raises the following questions that we aim to address:
\begin{itemize}[nosep, leftmargin=*]
    \item Does the convex-like shape hold for any image pair?

    \item If so, why? Is there any guarantee on solutions?
\end{itemize}

		
		

\bfsection{Our Approach: Dee\underline{p} Sta\underline{r}-Convexif\underline{i}ed Luca\underline{s}-Kanad\underline{e} (PRISE)}
To mitigate the issue of poor local optimality in homography estimation, in this paper we propose a novel approach, namely PRISE, to enforce deep neural networks to approximately learn star-convex loss landscapes for the downstream tasks. Recently star-convexity \cite{nesterov2006cubic} in nonconvex optimization has been attracting more and more attention \cite{lee2016optimizing, pmlr-v125-hinder20a, gower2021sgd, kuruzov2021sequential} because of its capability of finding near-optimal solutions based on GD with theoretical guarantees. Star-convex functions refer to a particular class of (typically) non-convex functions whose global optimum is visible from every point in a downhill direction. From this view, convexity is a special case of star-convexity. In the literature, however, most of the works focus on optimizing and analyzing star-convex functions, while learning such functions is hardly explored.
In contrast, our PRISE imposes additional hinge losses, derived from the definition of star-convexity, on the learning objective during training. At test time, the nice convergence properties of star-convexity help find provably near-optimal solutions for the tasks using the LK method. We further show that DeepLK is a simplified and approximate algorithm of PRISE, and thus shares some properties with ours, but with worse performance.

Recently \cite{zhou2018sgd} have shown that stochastic gradient descent (SGD) will converge to global minimum in deep learning if the assumption of star-convexity in the loss landscapes hold. They validated this assumption (in a major part of training processes) empirically using relatively shallow networks and small-scale datasets by showing the classification training losses can converge to zeros. Nevertheless, we argue that this assumption may be too strong to hold in complex networks for challenging tasks, if without any additional imposition on learning. In our experiments we show that even we attempt to learn star-convex loss landscapes, the outputs at both training and test time are hardly perfect for complicated tasks.

\bfsection{Contributions} Our key contributions are listed as follows:
\begin{itemize}[nosep, leftmargin=*]
    \item We propose a novel PRISE method for multimodel image alignment by introducing (strongly) star-convex constraints into the network training, which is rarely explored in the literature of deep learning.
    \item We provide some analysis on the quality of the solutions from PRISE through star-convex loss landscapes.
    \item We demonstrate the state-of-the-art results on some benchmark datasets for multimodel image alignment with much better accuracy, especially when the pixel errors are small.
\end{itemize}

\section{Related Work}
\bfsection{Homography Estimation}
Homography estimation is a classic task in computer vision. The feature-based methods \cite{ye2017robust, fu2019adaptive, ye2019fast} have existed for several decades but required similar contextual information to align the source and target images. To overcome this problem, researchers use deep neural networks \cite{erlik2017homography, nguyen2018unsupervised, le2020deep, zhang2020content} to increase the alignment robustness between the source and template images.
For instance, DHM \cite{detone2016deep} produces a distribution over quantized homographies to directly estimates the real-valued homography parameters. MHN \cite{le2020deep} utilizes a multi-scale neural network to handle dynamic scenes. Since then, finding a combinatorial method from classical and deep learning approaches has become possible. Recent models such as CLKN \cite{chang2017clkn}, DeepLK \cite{zhao2021deep} focus on learning a feature map for traditional Inverse Compositional Lucas-Kanade method on multimodal image pairs. Also, IHN \cite{Cao_2022_CVPR} provides a correlation finding mechanism and iterative homography estimators across different scale to improve the performance of homography estimation without any untrainable part. A good survey can be found in \cite{agarwal2005survey}. 

\bfsection{Nonconvexity and Convexification} 
Nonconvexity is challenging in statistical learning where researchers proposed several regularized estimators \cite{tuan2000low, loh2013regularized, loh2015regularized} that can solve this issue partially. For deep learning or network training, such nonconvexity also brings serious trouble in optimization such as Adam \cite{kingma2014adam}. Recently, the concept of convexification has started to be introduced into the training process \cite{yang2020graduated}. Several works \cite{wang2020adaptively, mao2016successive, reddi2018adaptive,vettam2019regularized} have demonstrated that the convex properties can be utilized in training a deep neural network whose loss landscape shows nonconvexity. 

\bfsection{Adversarial Training}
Adversarial training is one of the most effective strategies for improving robustness with adversarial data generation and model training. 
For the former, generative adversarial network (GAN) \cite{goodfellow2020generative} and its variants \cite{creswell2018generative} are classic deep neural networks that can be used to generate adversarial examples. For the latter, fast gradient sign method (FSGM) \cite{goodfellow2014explaining} and its variants \cite{dong2018boosting, lin2019nesterov, andriushchenko2020understanding} are widely used to train deep models. For instance, Shafahi \etal \cite{shafahi2019adversarial} proposed an algorithm that eliminates the overhead cost of generating adversarial examples by recycling the gradient information computed when updating model parameters. Wong \etal \cite{wong2020fast} demonstrated that it is possible to train empirically robust models using a much weaker and cheaper adversary. Good surveys can be found in \cite{bai2021recent, qian2022survey}.

\bfsection{Contrastive Learning}
Recently, learning representations from unlabeled data in contrastive way \cite{chopra2005learning, hadsell2006dimensionality} has been one of the most competitive research field \cite{oord2018representation, hjelm2018learning, wu2018unsupervised, tian2020contrastive, sohn2016improved, chen2020simple, jaiswal2020survey, li2020prototypical, he2020momentum, chen2020improved, chen2020big, bachman2019learning, misra2020self, caron2020unsupervised}. Popular methods such as SimCLR \cite{chen2020simple} and Moco \cite{he2020momentum} apply the commonly used loss function InfoNCE \cite{oord2018representation} to learn latent representation that is beneficial to downstream tasks. Several theoretical studies show that contrastive loss optimizes data representations by aligning the same image's two views (positive pairs) while pushing different images (negative pairs) away on the hypersphere \cite{wang2020understanding, chen2021intriguing, wang2021understanding, arora2019theoretical}. In terms of applications there are a large amount of works in images \cite{zimmermann2021contrastive, tian2020makes, li2020prototypical, he2020momentum} and 3D point clouds \cite{tang2022contrastive, yang2021unsupervised, jiang2021guided, du2021self, wang2022improving, yan2022implicit, eckart2021self, afham2022crosspoint, shao2022scrnet}, just to name a few. A good survey can be found in \cite{le2020contrastive}.

		
		

\section{Deep Star-Convexified Lucas-Kanade}
\subsection{Preliminaries}
\subsubsection{Homography Estimation}\label{sssec:homography}
Homography refers to a mapping between two planar projections of an image whose parameters are represented by a 3$\times$3 transformation matrix in a homogenous coordinates space and need to be estimated. The LK method is one of the classic algorithms in computer vision for homography estimation between images. Its nonconvex objective can be formulated as follows:
\begin{align}\label{eqn:Lucas-Kanade}
    \min_{\omega\in\Omega} \|x_t - g(x_s; \omega)\|_F^2,
\end{align}
where $x_s, x_t\in\mathcal{I}$ denote a source and target input images (equivalent to $x=\{x_s, x_t\}$ in Eq. \ref{eqn:downstream_task}), $\omega\in\Omega\subseteq\mathbb{R}^{3\times3}$ denotes the homography parameters, $g:\mathcal{I}\times\Omega\rightarrow\mathcal{I}$ denotes a nonconvex warping function, and $\|\cdot\|_F$ is the Frobenius norm. The LK algorithm uses GD to optimize Eq.~\ref{eqn:Lucas-Kanade}. 

\subsubsection{DeepLK}
Recently Zhao \etal \cite{zhao2021deep} proposed a deep learning based LK method (DeepLK) that essentially rewrites Eq. \ref{eqn:Lucas-Kanade} as follows:
\begin{align}\label{eqn:deep-Lucas-Kanade}
    \min_{\omega\in\Omega} \|f_t(x_t; \theta_t^*) - g(f_s(x_s; \theta_s^*); \omega)\|_F^2,
\end{align}
where functions $f_s:\mathcal{I}\times\Theta_s\rightarrow\mathcal{I}, f_t:\mathcal{I}\times\Theta_t\rightarrow\mathcal{I}$ denote two deep neural networks parametrized by the learned $\theta_s^*\in\Theta_s, \theta_t^*\in\Theta_t$, respectively (equivalent to $f=\{f_s, f_t\}, \theta^*=\{\theta_s^*, \theta_t^*\}, \Theta = \Theta_s \bigcup \Theta_t$ in Eq. \ref{eqn:downstream_task}), which transfer the source and target images into another two images. Then the original LK method can be directly applied to such transferred images for homography estimation with no change. 

\subsubsection{Star-Convexity}
\begin{defi}[Star-Convexity \cite{lee2016optimizing}]
A function $f:\mathbb{R}^n\rightarrow\mathbb{R}$ is {\em star-convex} if there is a global minimum $\omega^*\in\mathbb{R}^n$ such that for all $\lambda\in [0, 1]$ and $\omega\in\mathbb{R}^n$, it holds that
\begin{align}\label{eqn:sc_def1}
    f((1-\lambda) \omega^* + \lambda \omega) &\leq  (1-\lambda) f(\omega^*) + \lambda f(\omega).
\end{align}
\end{defi}

\begin{defi}[Strong Star-Convexity \cite{pmlr-v125-hinder20a}]
A differentiable function $f:\mathbb{R}^n\rightarrow\mathbb{R}$ is {\em $\mu$-strongly star-convex} with constant $\mu>0$ if there is a global minimum $\omega^*\in\mathbb{R}^n$ such that for $\forall \omega\in\mathbb{R}^n$, it holds 
\begin{align}\label{eqn:ssc_def1}
    f(\omega^*) \geq f(\omega) + \nabla f(\omega)^T(\omega^* - \omega) + \frac{\mu}{2}\|\omega^* - \omega\|_2^2, 
\end{align}
where $\nabla$ denotes the (sub)gradient operator, $(\cdot)^T$ denotes the matrix transpose operator, and $\|\cdot\|_2$ denotes the $\ell_2$ norm. Note when $\mu=0$ Eq. \ref{eqn:ssc_def1}, will become equivalent to Eq.~\ref{eqn:sc_def1}.
\end{defi}

\begin{lemma}\label{lem:1}
The following conditions hold iff
a function $f:\mathbb{R}^n\rightarrow\mathbb{R}$ is $\mu$-strongly star-convex, given a global minimum $\omega^*\in\mathbb{R}^n$ and $\forall \lambda\in[0,1], \forall \omega\in\mathbb{R}^n$: 
\begin{align}
    & f(\omega^*) \leq f(\Tilde{\omega}) - \frac{\mu}{2}\|\omega^*-\Tilde{\omega}\|_2^2, \label{eqn:lem-1} \\
    & f(\Tilde{\omega}) \leq (1-\lambda) f(\omega^*) + \lambda f(\omega) - \frac{\lambda(1-\lambda)\mu}{2}\|\omega^*-\omega\|_2^2, \label{eqn:lem-2}
\end{align}
where $\Tilde{\omega} = (1-\lambda) \omega^* + \lambda \omega$.
\end{lemma}
\begin{proof}
As illustrated in Fig. \ref{fig:sc-demo}, a cut through $\omega^*, \omega$ forms a convex shape if $f$ is star-convex. Therefore, since $\nabla f(\omega^*) = \mathbf{0}$, Eq. \ref{eqn:ssc_def1} will lead to Eq. \ref{eqn:lem-1} by replacing $\omega$ with $\Tilde{\omega}$ when switching the notations of $\omega^*, \omega$ in the equation. 

Letting $g(\omega) = f(\omega) - \frac{\mu}{2}\|\omega\|_2^2$, based on Eq.~\ref{eqn:ssc_def1} we have $g(\omega^*)\geq g(\omega) + \nabla g(\omega)^T(\omega^* - \omega)$, \ie $g$ is star-convex. Then based on $g$ and Eq. \ref{eqn:sc_def1}, we can achieve Eq.~\ref{eqn:lem-2}.
\end{proof}

\setlength{\columnsep}{10pt}%
\begin{wrapfigure}{r}{.37\linewidth}
\centering 
\vspace{-5mm}
\includegraphics[width=1\linewidth]{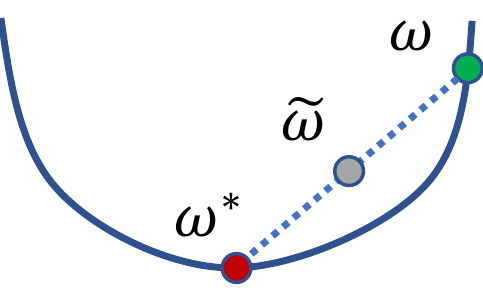}
\vspace{-5mm}
\caption{Geometric relations between $\omega^*, \omega, \Tilde{\omega}$.
}
\vspace{-3mm}
\label{fig:sc-demo}
\end{wrapfigure} 

Eq. \ref{eqn:lem-1} implies that $\omega^*$ will be a (local) minimum if it holds for $\forall \omega, \lambda$. In fact, Lemma \ref{lem:1} discusses the (tight) strong star-convexity {\em with no gradients}. In our approach we will use this lemma to incorporate the strong star-convexity as constraints in network training. 


\subsection{Approach}\label{ssec:approach}
\bfsection{Geometric Constraints}
Learning exact star-convex loss landscapes with the strong condition is challenging, even in very high dimensional spaces and in local regions. One common solution is to introduce slack variables as penalty to measure soft-margins. However, this may significantly break the smoothness of star-convex loss, and thus destroy the nice convergence property of gradient descent (see the results in Fig. \ref{fig:perf-condition} in our experiments). Therefore, in order to preserve the smoothness, we consider two geometric constraints that have capability to improve the loss landscape smoothness at different levels, just in case that one is too strong to be learned properly. Ordered by the strength of each geometric constraint from weak to strong, they are:
\begin{itemize}[nosep, leftmargin=*]
    
    \item {\em A strong star-convexity constraint in Eq. \ref{eqn:lem-1}:} This constraint implies that there exists a quadratic shape as the lower envelope of the loss landscape with minimum at $\omega^*$. Meanwhile, it also guarantees that the loss at the ground-truth $\omega^*$ on the (local surface of) loss landscape will reach the (local) minimum, as requested by star-convexity. 
    
    \item {\em A second strong star-convexity constraint in Eq. \ref{eqn:lem-2}:} This constraint imposes strong convexity on all the curves that connect $\omega^*$ with any other point on the loss landscape.
\end{itemize}

\setlength{\columnsep}{10pt}%
\begin{wrapfigure}{r}{.37\linewidth}
\centering 
\vspace{0mm}
\includegraphics[width=1\linewidth]{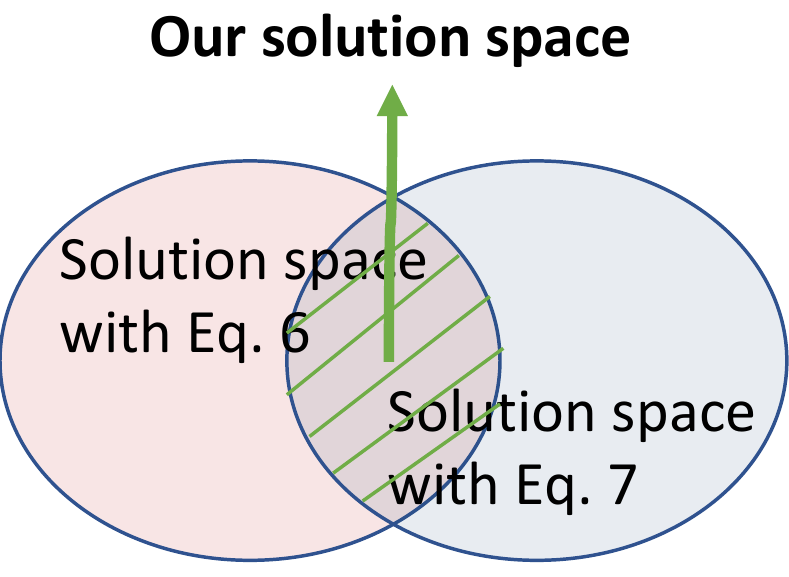}
\vspace{-5mm}
\caption{Illustration of solution spaces induced by the constraints.}
\vspace{-5mm}
\label{fig:solution_space}
\end{wrapfigure} 

The fundamental difference between the two strong star-convexity constraints is the positioning of $\nabla f$, where Eq. \ref{eqn:lem-1} is posited at $\omega^*$ while Eq. \ref{eqn:lem-2} is posited at $\omega$. As illustrated in Fig. \ref{fig:solution_space}, both constraints can lead to their own solution spaces for the same objective, but with an overlap where solutions will better approximate star-convexity.
This insight motivates our formulation.

\bfsection{Our Formulation}
Let $(x_i, \omega_i^*)$ denote a training sample with data $x_i$ (\eg image pairs for alignment) and its ground truth $\omega_i^*$. To simplify our notations, in the sequel we will denote $h_{\theta}(\cdot)=\ell_{f}(\cdot; x_i, \theta), \Tilde{\omega}_i = (1-\lambda) \omega_i^* + \lambda \omega_i$, and we wish to learn $h_{\theta}$ to be strongly star-convex. Now based on our considerations on geometric conditions, we are ready to formulate our learning framework as follows:
\begin{align} 
    \hspace{-0.5mm}\min_{\theta} & \sum_i \left\{h_{\theta}(\omega_i^*) + \rho\mathbb{E}_{\omega_i\sim\mathcal{N}_{\omega_i^*}}\Big\{\max_{\lambda\in[0,1]}\epsilon_{\omega_i} + \max_{\lambda\in[0,1]}\xi_{\omega_i}\Big\} \right\} \label{eqn:problem} \\
    \hspace{-0.5mm}\mbox{s.t.} \hspace{1.5mm}
    & h_{\theta}(\omega_i^*) \leq h_{\theta}(\Tilde{\omega}_i) - \frac{\mu}{2}\|\omega_i^* - \Tilde{\omega}_i\|_2^2 + \epsilon_{\omega_i}, \label{eqn:con1} \\
    & h_{\theta}(\Tilde{\omega}_i) \leq (1-\lambda)h_{\theta}(\omega_i^*) + \lambda h_{\theta}(\omega_i) \nonumber \\
    & \hspace{20mm} - \frac{\lambda(1-\lambda)\mu}{2}\|\omega_i^* - \omega_i\|_2^2 + \xi_{\omega_i}, \label{eqn:con2} \\
    & \forall\epsilon_{\omega_i}\geq0, \forall \xi_{\omega_i}\geq0, \forall i, \nonumber
\end{align}
where $\epsilon_{\omega_i}, \xi_{\omega_i}$ are the slack variables for hinge losses, $\mathcal{N}_{\omega_i^*}$ is a (local) neighborhood around $\omega_i^*$ where $\omega_i$ is sampled from, $\rho\geq0, \mu\geq0$ are predefined trade-off and surface sharpness parameters, respectively, and $\mathbb{E}$ denotes the expectation.

\begin{algorithm}[t]
    \SetAlgoLined
    \SetKwInOut{Input}{Input}\SetKwInOut{Output}{Output}
    \Input{training data $\{(x_i, \omega^*_i)\}$, LK loss function $h_{\theta}$ and a network architecture $f$, hyperparameters $\lambda, \mu, \rho$}
    \Output{network parameters $\theta^*$}
    \BlankLine
    Randomly initialize $\theta$;
            
    \Repeat{Converge or maximum number of iterations is reached}{
        
        Randomly select a training image pair with its ground truth $(x_i, \omega^*_i)$;
        
        Randomly sample (multiple) $\omega_i\sim\mathcal{N}_{\omega_i^*}$; 
        
        Update $\theta$ by solving Eq. \ref{eqn:problem} with strong star-convex constraints in Eqs. \ref{eqn:con1} and \ref{eqn:con2};
    }
    \Return $\theta^* \leftarrow \theta$;
    \caption{PRISE: Deep Star-Convexified Lucas-Kanade}\label{alg:prise}
\end{algorithm}

\bfsection{Contrastive Adversarial Training}
Our approach is highly related to adversarial training and contrastive learning, since during training we try to create new fake samples $\Tilde{\omega}_i, \omega_i$, compare their losses with $h_{\theta}(\omega^*)$, and solve a minimax problem defined in Eq. \ref{eqn:problem}. The contrastive learning comes from the nature of strong star-convexity, leading to extra hinge losses. The adversarial training starts from finding the values of $\lambda$ that return the maximum $\epsilon_{\omega_i}, \xi_{\omega_i}$, respectively. Note that here the values are allowed to be different for the two hinge losses. Both together aim to control the loss landscapes towards being star-convex.


\bfsection{Implementation}
Eqs. \ref{eqn:con1} and \ref{eqn:con2} define a large pool of inequalities for each data point with varying $\omega_i$ and $\lambda$, where any inequality returns a hinge loss. To leverage our computational cost, motivated by the training algorithm for DeepLK we propose a similar {\em sampling} based training algorithm, as listed in Alg. \ref{alg:prise}\footnote{In our code we use batch-based implementation for fast computation.}. Specifically,
\begin{itemize}[nosep, leftmargin=*]
    \item {\em Sampling from $\mathcal{N}_{\omega_i^*}$ for $\omega_i$:} Same as stochastic gradient descent (SGD), we sample a fixed number of $\omega_i$ for each ground truth, and then compute the average.

    \item {\em Sampling from $[0,1]$ for $\lambda$:} We could solve the maximum problems using FSGM with careful parameter tuning. However, this will introduce huge computational burden, as the complexity will be proportional to the number of samples for $\omega_i$ times the number of data points $(x_i, \omega_i^*)$. Therefore, to address the computational complexity issue, we instead simply take $\lambda$ as a predefined hyperparameter that are shared by $\epsilon_{\omega_i}, \xi_{\omega_i}$. We have evaluated the way of sampling multiple copies of $\lambda$ and then choosing the maximum hinge losses with no sharing for learning. We observe that the results are very similar to those with the predefined one, but the training time is much longer.
\end{itemize}
As a demonstration, we take the network in DeepLK as our backbone and use the same LK loss as DeepLK for training. At test time, we substitute the learned network weights $\theta^*$ into the right side of Eq. \ref{eqn:downstream_task} and solve the original non-convex image alignment problem using the LK method.


\bfsection{Star-Convexity \vs Convexity}
Star-convexity enforces to learn one-point convexity \cite{li2017convergence}, where we only impose the convexity at the ground truth within a local region. In contrast, convexity requires much more data points, making the training of deep models much less efficient. 

\subsection{Analysis}\label{ssec:analysis}
\bfsection{Relations to DeepLK}
In fact, DeepLK imposes the following two conditions on the minimization of the LK loss:
\begin{align}\label{eqn:deeplk}
    \left\{
    \begin{array}{l}
         h_{\theta}(\omega_i^*) \leq h_{\theta}(\omega_i) - \|\omega_i^* - \omega_i\|_2^2, \\
         h_{\theta}(\Tilde{\omega}_i) \leq h_{\theta}(\omega_i) - (1-\lambda^2)\|\omega_i^* - \omega_i\|_2^2.
    \end{array}
    \right. 
\end{align}
Note that when $\lambda=0$, these two inequalities will become the same. Below we will only discuss the lower equation in Eq. \ref{eqn:deeplk}. Then we have the following lemma:
\begin{lemma}\label{lem:equivalence}
    It holds for the RHS of Eqs. \ref{eqn:deeplk} and \ref{eqn:con2} that   
    \begin{align}
         (1-\lambda)h_{\theta}(\omega_i^*) + \lambda &h_{\theta}(\omega_i) - \frac{\lambda(1-\lambda)\mu}{2}\|\omega_i^* - \omega_i\|_2^2 \nonumber \\
        \underset{}{\overset{\mu\geq2}{\leq}} & h_{\theta}(\omega_i) - (1-\lambda^2)\|\omega_i^* - \omega_i\|_2^2,
    \end{align}
    with the same weights $\theta$ and the equality holds when $\mu=2$.
\end{lemma}
\begin{proof}
    With the help of Eq. \ref{eqn:lem-1} and simple algebra, we can easily prove this lemma.
\end{proof}
From this lemma, we can see that DeepLK potentially explores a larger solution space than our strong star-convexity. In other words, every solution returned from our PRISE would fall in the solution space of DeepLK. Therefore, we hypothesize that DeepLK may be able to learn some loss landscapes close to star-convex

\begin{figure*}[t]
\centering
\includegraphics[width=0.4\linewidth]{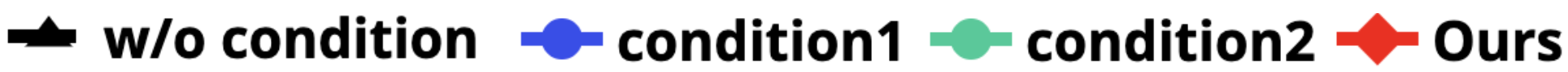}
 
	\begin{minipage}[b]{0.325\textwidth}
		\centering
			\centerline{\includegraphics[width=1.0\linewidth, keepaspectratio,]{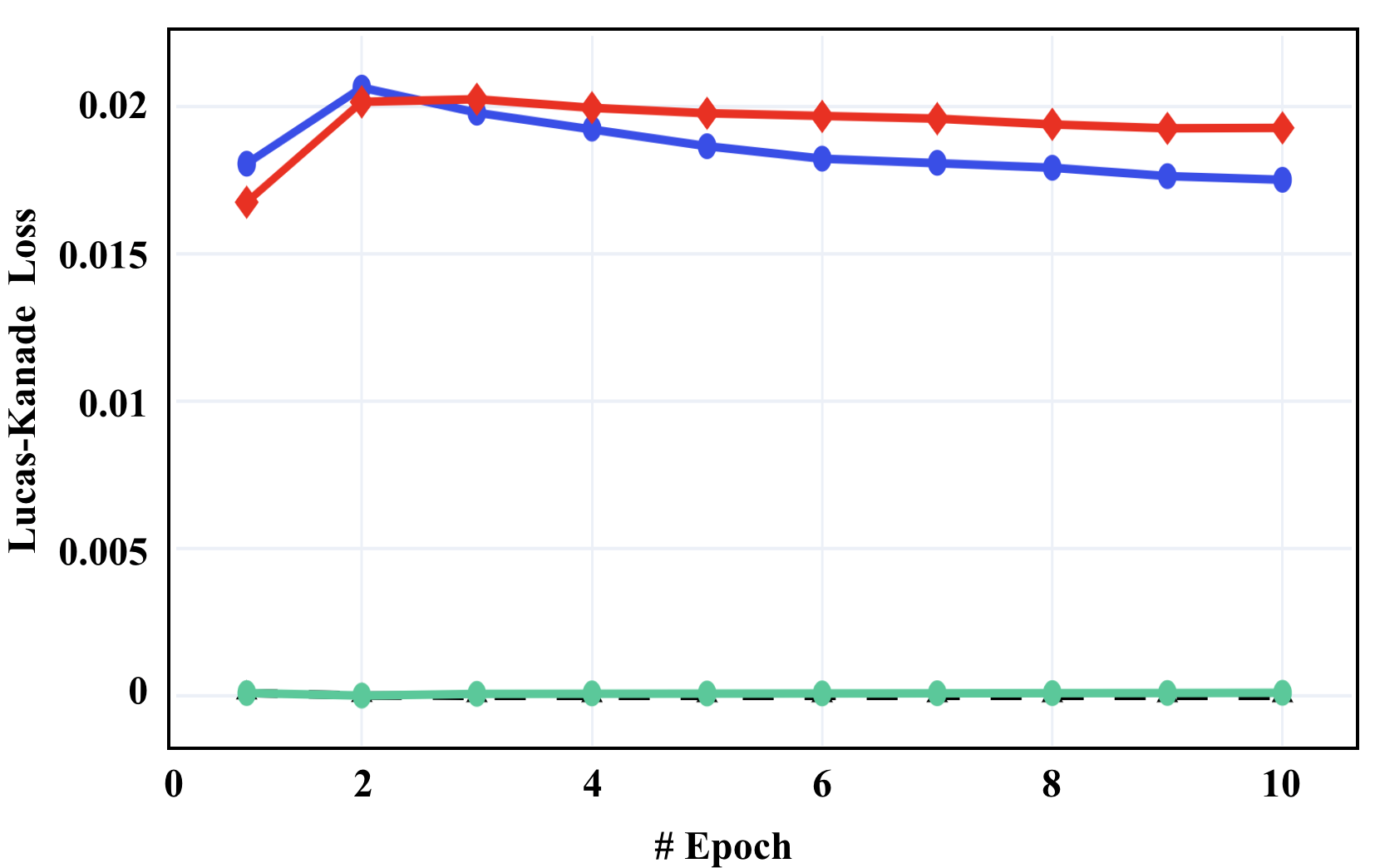}}
			\subcaption{LK loss}
		
	\end{minipage}
	\hfill
	\begin{minipage}[b]{0.325\textwidth}
		\centering
			\centerline{\includegraphics[width=1.0\linewidth,keepaspectratio]{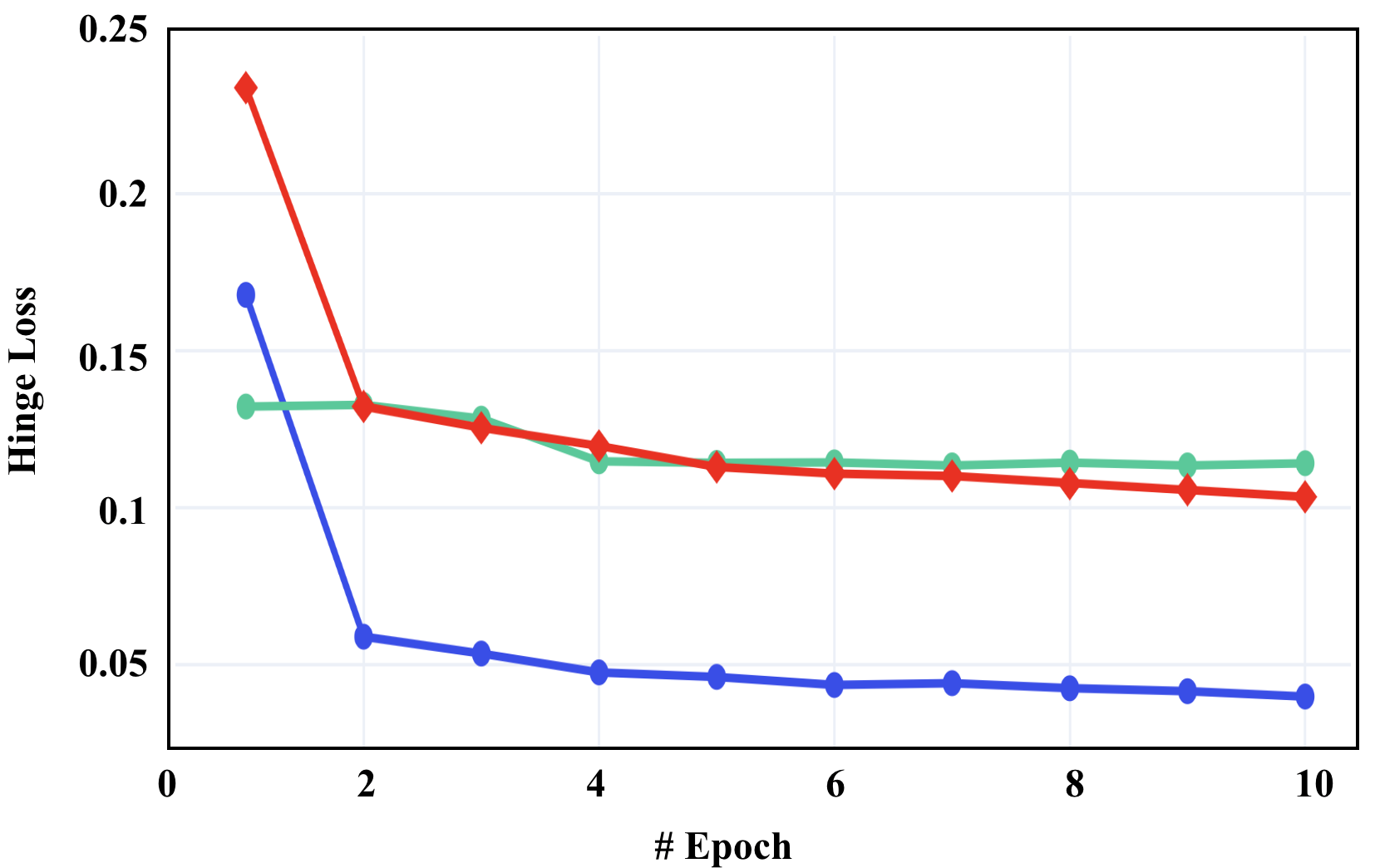}}
			\subcaption{Hinge loss}
		
	\end{minipage}
	\hfill
	\begin{minipage}[b]{0.325\textwidth}
		\centering
			\centerline{\includegraphics[width=1.0\linewidth,keepaspectratio]{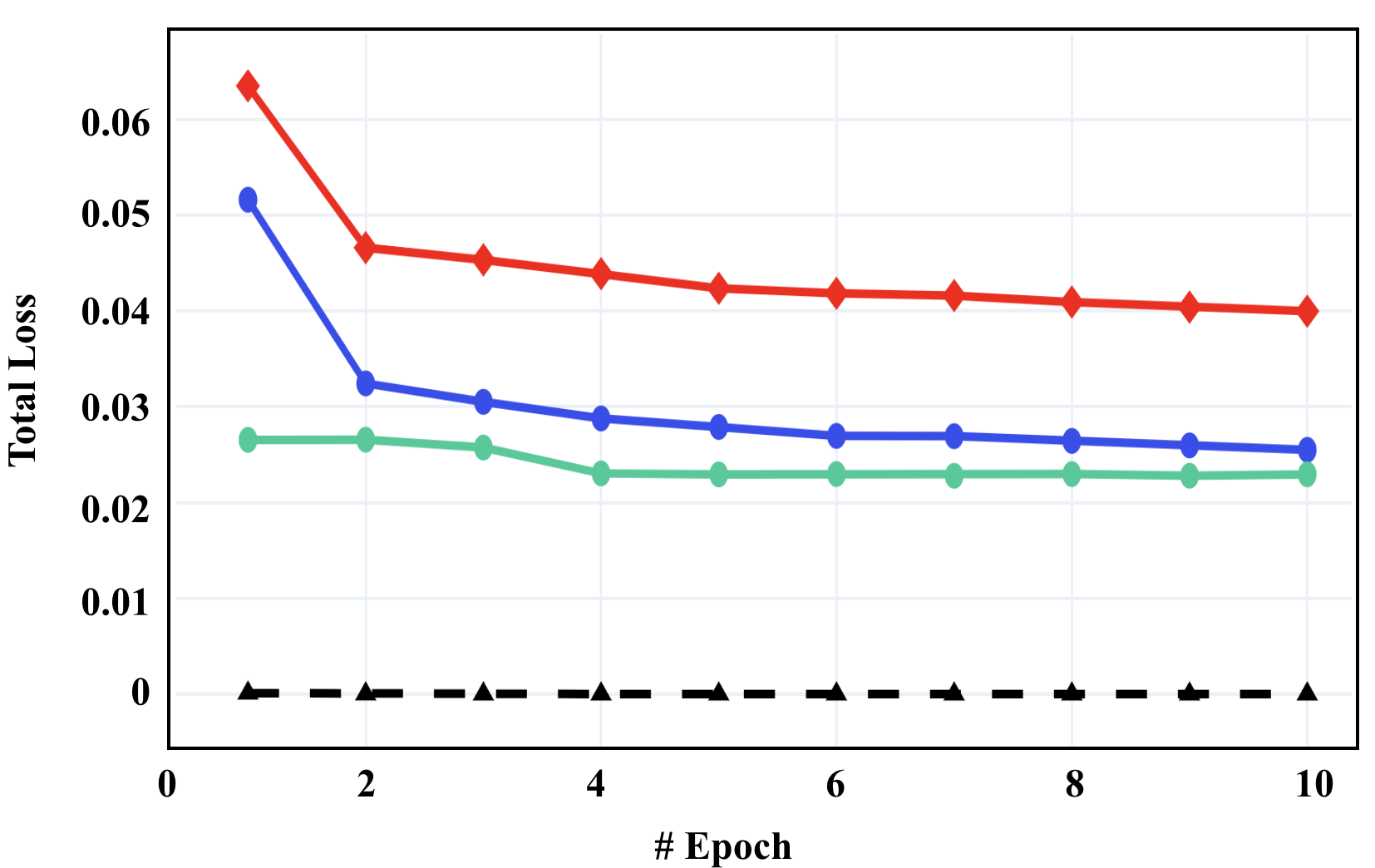}}
			\subcaption{Total loss}
		
	\end{minipage}
    \caption{Comparison on the training loss at Stage 3 on GoogleEarth.}
\label{fig:loss}
\end{figure*}

\begin{figure}[t] 
\centering
\includegraphics[width=0.8\linewidth]{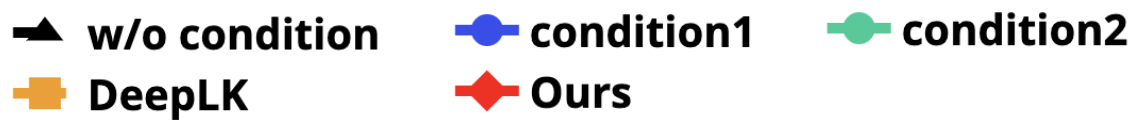}
\centering 
\includegraphics[width=1\columnwidth]{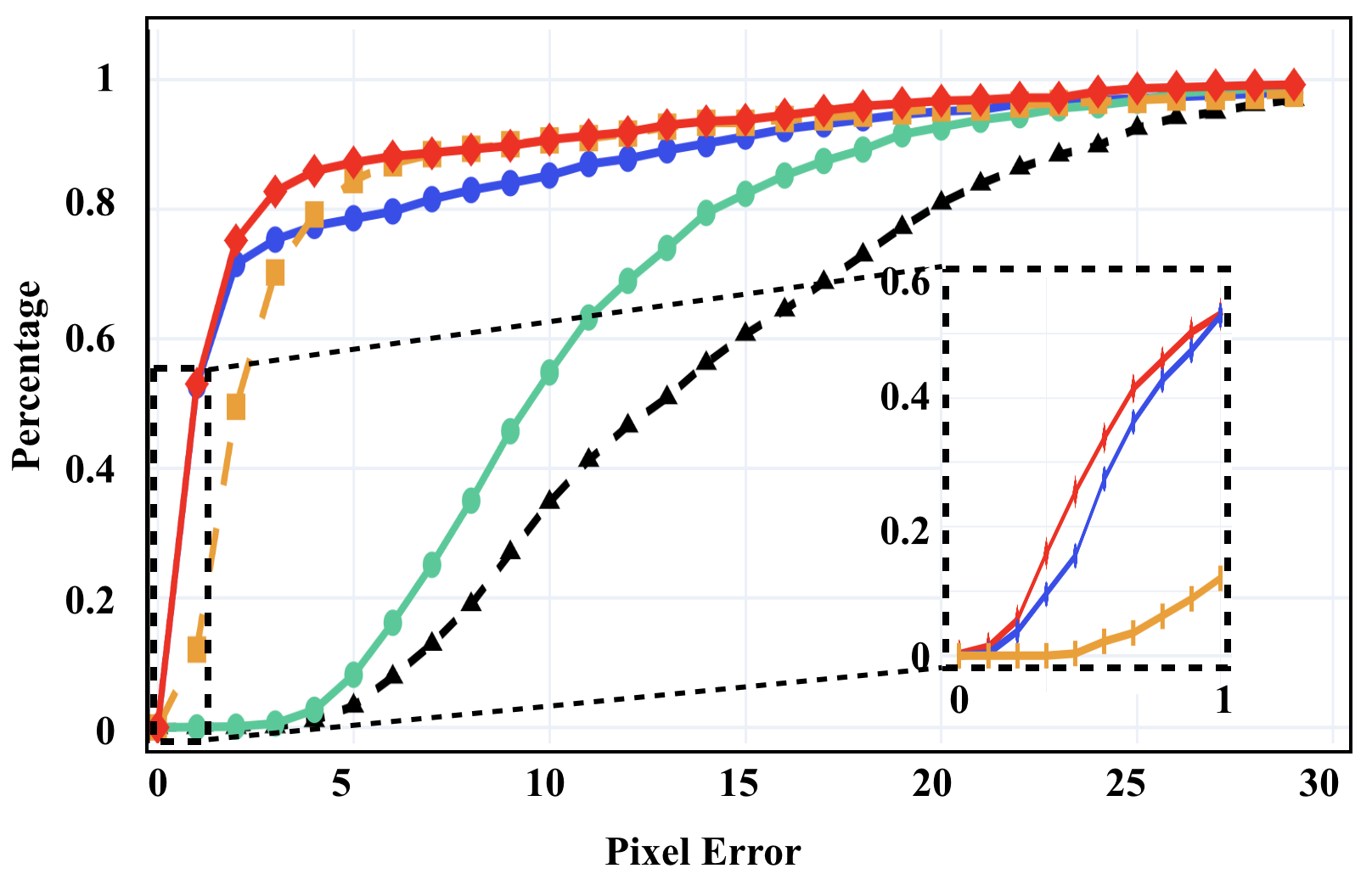}
\vspace{-3mm}
\caption{Performance comparison on GoogleEarth using the default hyperparameter setting.} 
\vspace{-3mm}
\label{fig:perf-condition}
\end{figure} 

\begin{figure*}[t]
	\begin{minipage}[b]{0.325\textwidth}
		\centering
			\centerline{\includegraphics[width=1.0\linewidth, keepaspectratio,]{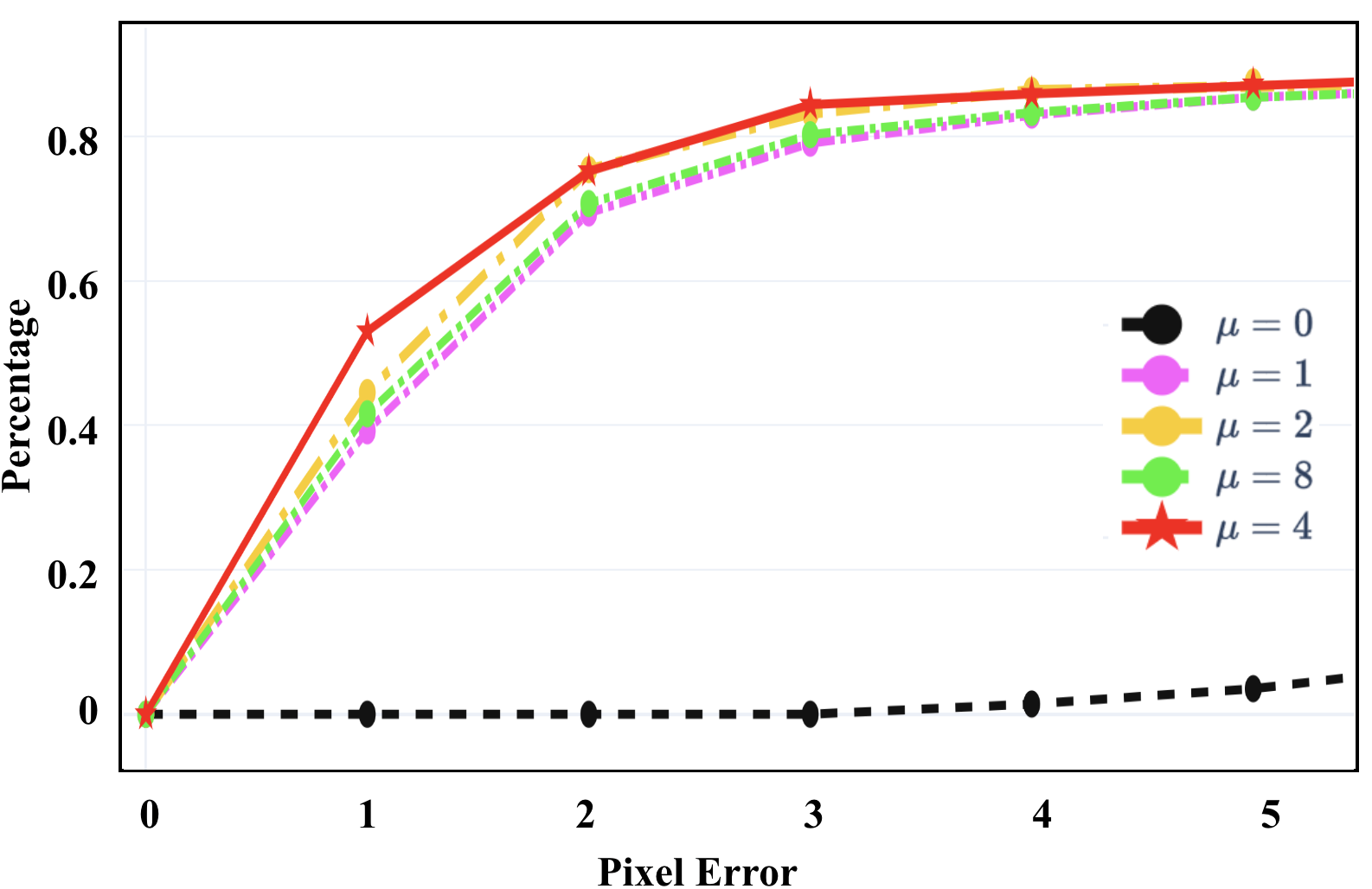}}
			\subcaption{}
	\end{minipage}
	\hfill
	\begin{minipage}[b]{0.325\textwidth}
		\centering
			\centerline{\includegraphics[width=1.0\linewidth,keepaspectratio]{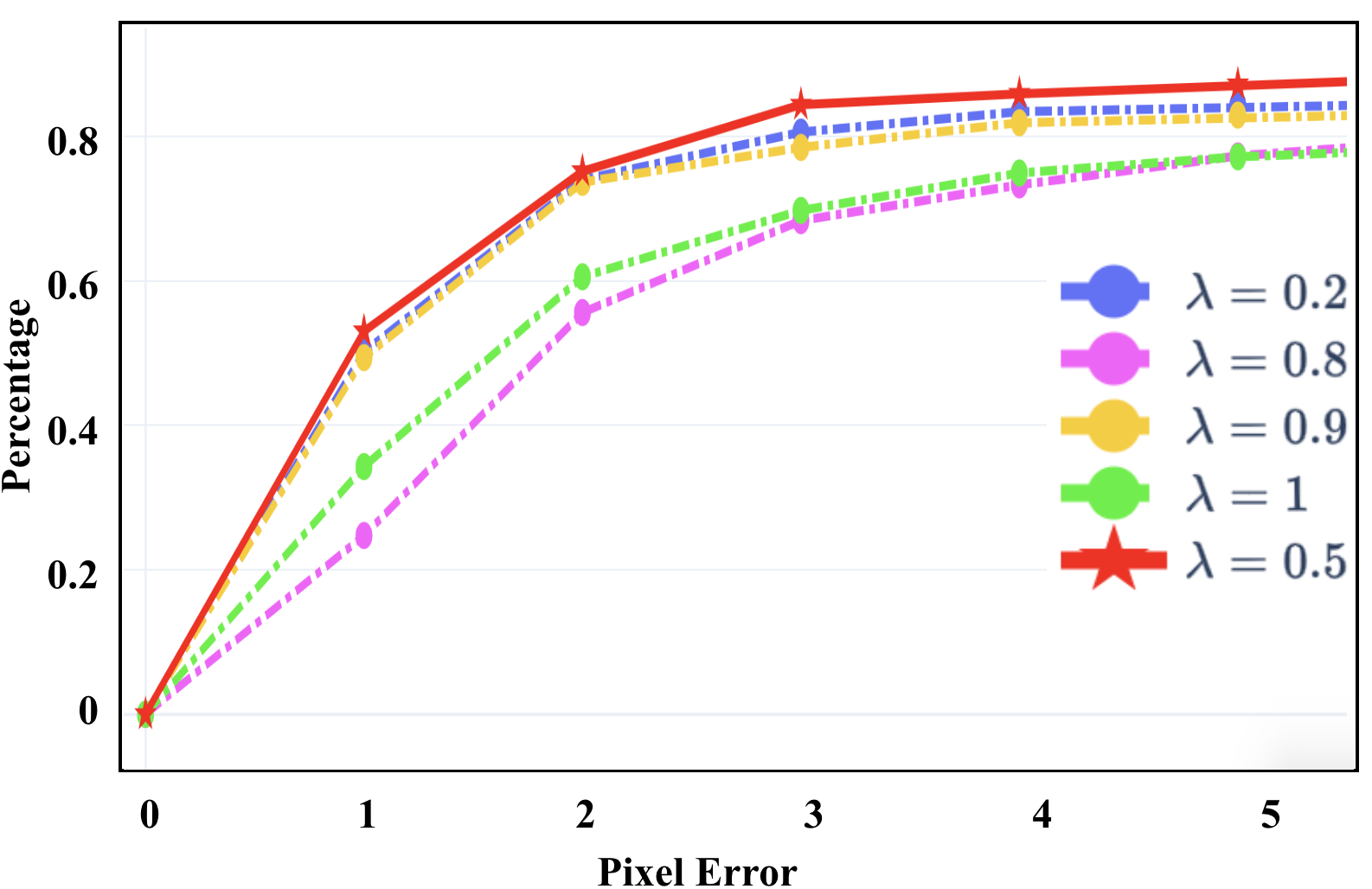}}
			\subcaption{}
	\end{minipage}
	\hfill
	\begin{minipage}[b]{0.325\textwidth}
		\centering
			\centerline{\includegraphics[width=1.0\linewidth,keepaspectratio]{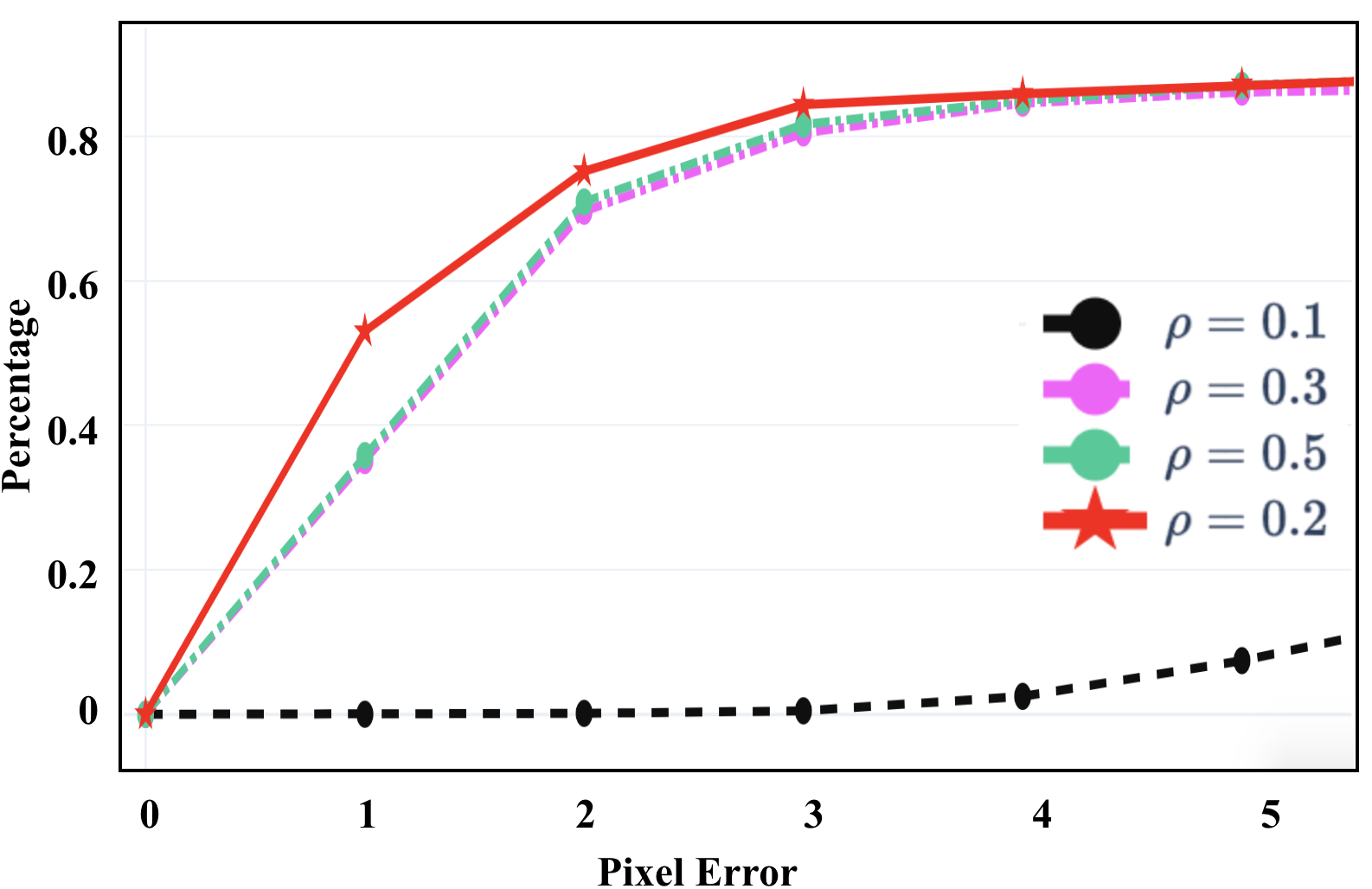}}
			\subcaption{}
	\end{minipage}
    \caption{Pixel error \vs various hyperparameters in Eqs. \ref{eqn:problem}-\ref{eqn:con2} on GoogleEarth. 
    }
\label{fig:PE-param}
\end{figure*}

\begin{figure}[t] 
\centering
\includegraphics[width=0.7\linewidth]{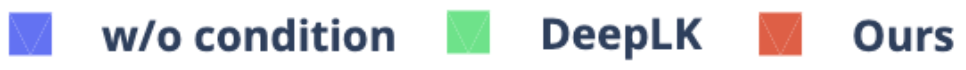}
\centering 
\includegraphics[width=.9\columnwidth]{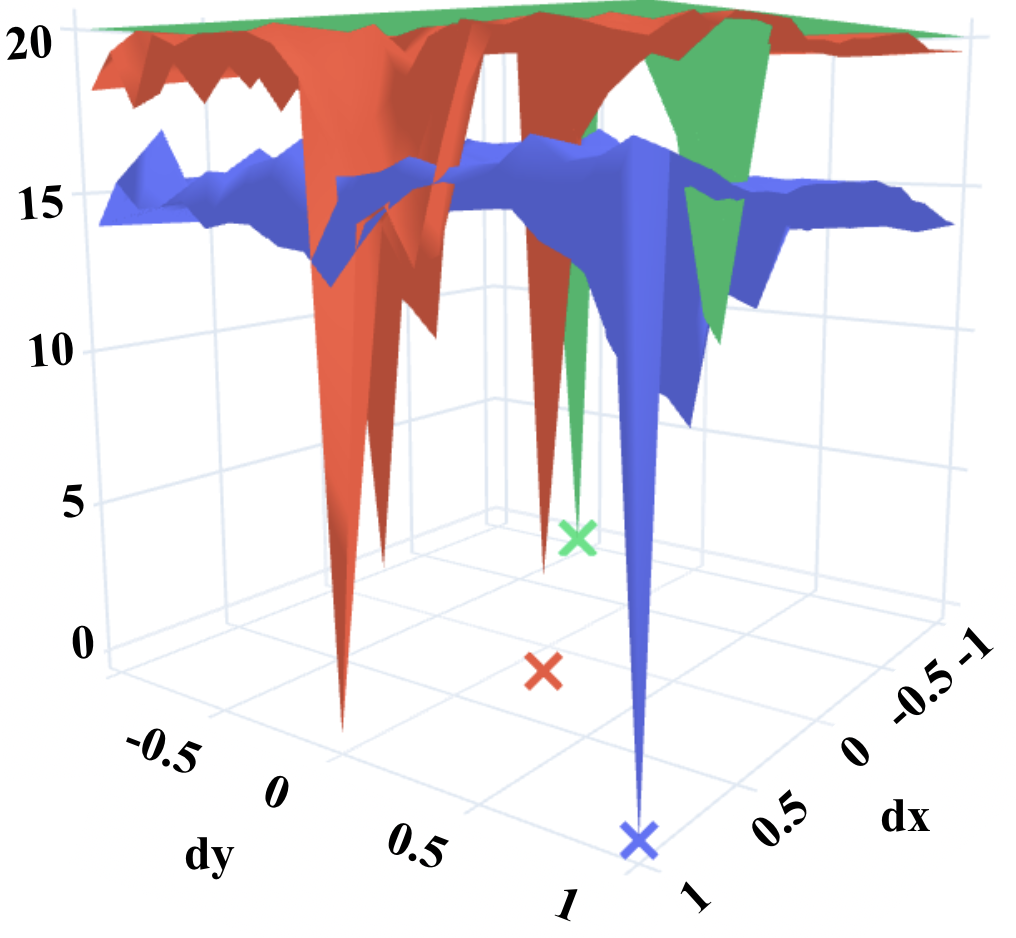}
\caption{Visualization of the loss landscapes of the plain model, DeepLK and PRISE by using a test pair of source image and target image randomly selected from GoogleEarth, and the $\times$'s on the dx-dy plane denote the locations of the homography estimation results from the three methods around the ground truth at $(0,0)$.} 
\label{fig:loss-landscape}
\end{figure} 

\begin{figure}[t] 
\centering 
\includegraphics[width=1\columnwidth]{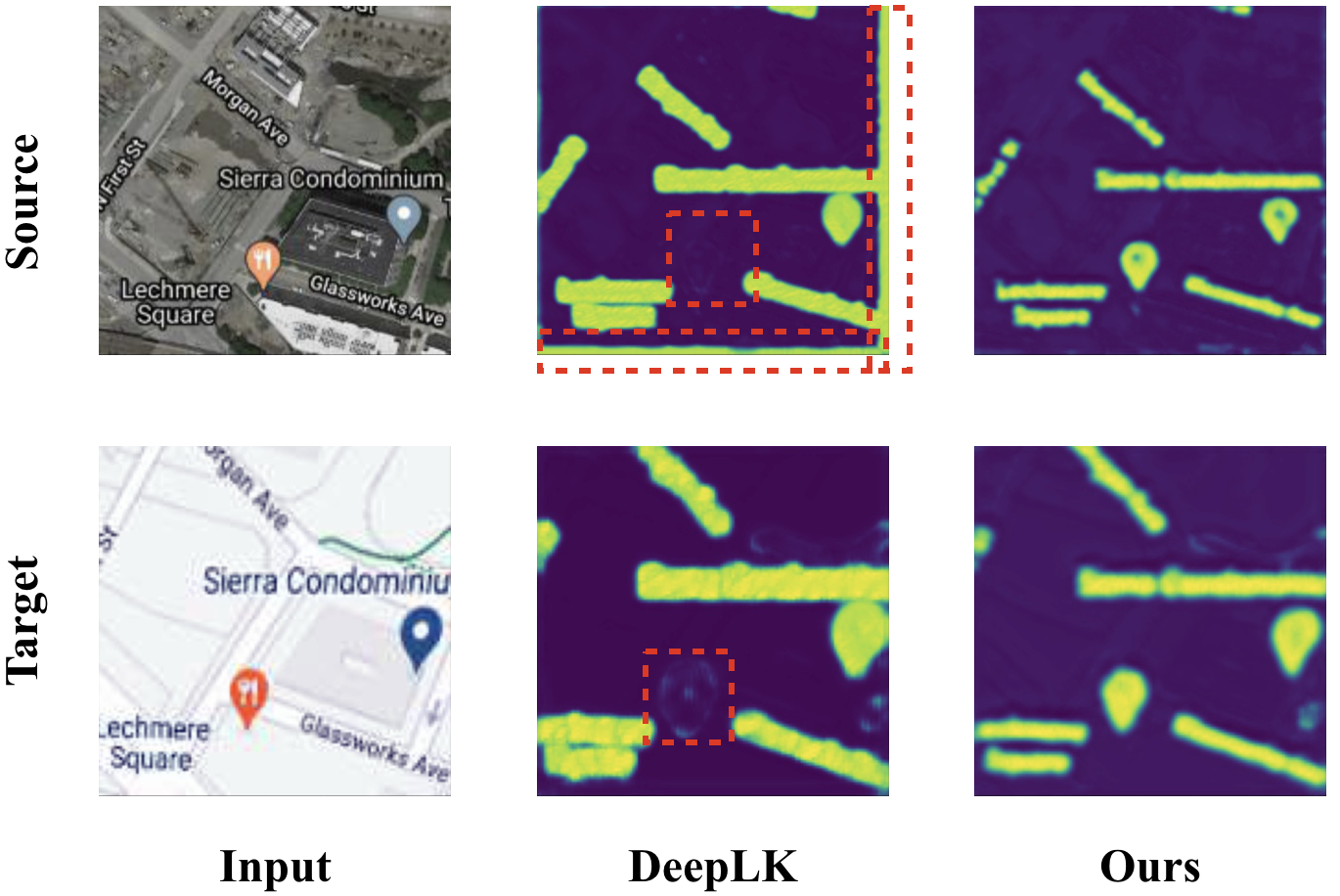}
\vspace{-3mm}
\caption{Feature map comparison on GoogleMap dataset, with boxes emphasizing the differences between DeepLK and ours.} 
\label{fig:perf-condition}\vspace{-5mm}
\end{figure}

\bfsection{Near-Optimal Solutions}
\cite{pmlr-v125-hinder20a} has shown that the GD based algorithms can find near-optimal solutions for the optimization of star-convex functions. To better see this intuitively, we can have the following inequality based on Eq. \ref{eqn:lem-1}:
\begin{align}\label{eqn:w_dis}
    \|\bar\omega^* - \bar{\omega}\|^2 \leq \frac{2}{\mu}\Big[h_{\theta^*}(\bar{\omega}) - h_{\theta^*}(\bar\omega^*)\Big],
\end{align}
where $\bar\omega^*$ denotes the ground truth for a test data point $\bar x$, $\bar\omega$ denotes the prediction based on a network with learned weights $\theta^*$, and $\mu\geq0$ is a constant. For image alignment with the LK loss, $h_{\theta^*}(\bar\omega^*) = 0$ and thus a smaller $h_{\theta^*}(\bar{\omega})$ implies a better solution, which is minimized using the LK loss in training. In fact, however, the distance is upper-bounded by the {\em contrastive loss}, which contributes to the hinge losses that are appended to the LK loss in our formulation. This observation indicates that the contribution of the hinge losses to a well-trained network may be higher than the LK loss (see Fig. \ref{fig:loss} for more details).

\section{Experiments}\label{sec:exp}



\bfsection{Datasets}
We exactly follow the experimental settings in DeepLK \cite{zhao2021deep}. We select an image from each dataset and resize it to $196\times 196$ pixels as input. Then we randomly perturb 4 points in the four corner boxes with size of $64\times64$ and resize the chosen region to a $128\times 128$ template image. We implement the same data generation strategy on three different datasets as follows:
\begin{itemize}[nosep, leftmargin=*]
    \item {\em MSCOCO \cite{lin2014microsoft}:} This is a benchmark dataset in computer vision, including homography estimation, with various foreground and background information. 6K images are sampled from the validation set as our test set.
    
    \item {\em GoogleEarth \cite{zhao2021deep}:} This is a high-resolution cross-season satellite dataset consisting of about 8K training images and 850 test images. Homography estimation on this dataset is challenging as the textural differences between different images are small (compared with the natural images in MSCOCO), and they can easily confuse the LK algorithm which tries to catch the change in grey scale. It has been demonstrated in DeepLK that many recent works on homography estimation failed on this dataset.
    
    \item {\em Google Maps and Satellite (GoogleMap for short) \cite{zhao2021deep}:} This dataset provides multimodal inputs for query and template images, consisting of about 8K training image pairs from Google Map and Google Satellite at the same locations and 888 test pairs. We use the satellite images as the queries and the google map images as the templates to find the homography change from the satellite data to map data. Many models such as DHN \cite{detone2016deep} and MHN \cite{le2020deep} failed to work on this dataset.
\end{itemize}

\bfsection{Baseline Algorithms} We train DHM \cite{detone2016deep} and MHN \cite{le2020deep} from scratch with the best hyperparameters. We use the pretrained model for DeepLK\cite{zhao2021deep} directly. In addition, we use the pretrained models for CLKN \cite{chang2017clkn} and fine-tune it on MSCOCO, GoogleEarth, and GoogleMap to fix the domain gap. Also we compare our approach with a classical algorithm SIFT+RANSAC.

\bfsection{Training \& Testing Protocols}
Following the consistent setting for each dataset in both training and testing, we use the same resolution for the source and target images as the standard datasets. For training, we train our model with best hyperparameters on each dataset for 10 times with random initialization of network weights. For testing, we conduct evaluation on the PEs under different thresholds. We report our results in terms of mean and standard deviation over 10 trials. All the experiments are done using an Nvidia RTX6000 GPU server

\bfsection{Our Implementation Details}
To demonstrate the improvement of our approach, we modify the implementation of DeepLK. Specifically,
\begin{itemize}[nosep, leftmargin=*]
    \item {\em Network architecture:} We employ the same network architecture in DeepLK. 
    The network has three identical stages with the same sub-networks. In each stage, it is a siamese design sharing weight with the source and target images. There are 3 residual blocks in each siamese architecture with 64 convolutional filters. The network downsamples the output feature maps using the stride of 2.
    
    \item {\em Training:} We train the network in a stacked way (\ie Stage 1 first, then Stage 2 and finally Stage 3), not end-to-end. That is, once the previous stage is fully trained, the current stage will start training using the output feature maps from the previous stage as input. We train each stage in the network for 10 epochs with a batch size of 4, a constant learning rate $10^{-5}$, and weight decay $0.05$. Except $\mu, \lambda, \rho$ that are determined by grid search, all the rest hyperparameters keep the same as used in DeepLK. We employ Adam \cite{kingma2014adam} as our optimizer. 
    By default, we use $\mu=2, \lambda=0.9, \rho=0.1, \#sample=2$ for MSCOCO, $\mu=4, \lambda=0.5, \rho=0.2, \#sample=4$ for GoogleEarth, and $\mu=2, \lambda=0.5, \rho=0.2, \#sample=4$ for GoogleMap, where $\#sample$ denotes the number of samples for $\omega_i$ (see Alg. \ref{alg:prise}).
    
    \item {\em Inference:} Our inference follows the coarse-to-fine strategy that fits the iterative updated methods in the classic LK algorithm. That is, both source and target images go through the learned network to extract three feature maps, one per stage, and then the classic LK algorithm is applied to each feature map, starting from Stage 3 up to Stage 1 in a backward manner. The homography parameters will be upscale with a factor of 2 sequentially since the resolutions of feature maps are different, until Stage 1 is done. MHN \cite{le2020deep} is used for initializing homography parameters.

\end{itemize}

\bfsection{Evaluation Metric}
We use the same evaluation metric as in recent works, Success Rate (SR) \vs Pixel Error (PE), to compare the performance of each algorithm. PE measures the average $L_2$ distance between the 4 ground-truth perturbation points and the 4 output point location predictions (without quantization) from an algorithm, correspondingly. Then the percentages of the testing image pairs whose PEs are smaller than certain thresholds, \ie SR, are computed to compare the performance of different approaches.



\subsection{Ablation Study}

\subsubsection{Effects of Strong Star-Convexity Constraints}
Fig. \ref{fig:loss} illustrates our comparisons on the training loss at Stage 3, where (a) shows the loss for the LK algorithm, (b) shows the combination loss from the extra hinge losses, and (c) shows the weighted sum of the two losses in (a) and (b) as our objective in Eq. \ref{eqn:problem}. Here we refer to Eq. \ref{eqn:con1} as Condition~1, and Eq. \ref{eqn:con2} as Condition~2. From this figure, we can see that
\begin{itemize}[nosep, leftmargin=*]
    \item The LK loss and hinge loss are balanced well at the same scale, no dominance scenarios occurring. All the losses tend to converge over the epochs.
    
    \item From Fig. \ref{fig:loss} (c), without any hinge loss, training with the original LK loss alone is very easy to overfit.
    
    \item Condition 1 alone seems much more important than Condition 2, because from Fig. \ref{fig:loss} (a) it can prevent the loss from overfitting (while Condition 2 cannot) and from Fig.~\ref{fig:loss} (b) it can lead to a much lower loss.    
\end{itemize}
Such observations partially support our analysis in Sec. \ref{ssec:analysis}, where narrower solution spaces seem not to help regularize the learning (\ie Lemma \ref{lem:equivalence}), and the contrastive loss contributes more to the learning (\ie Eq. \ref{eqn:w_dis}) as it leads to reduce the total loss, even though the LK loss increases.

We illustrate our performance comparison in Fig. \ref{fig:perf-condition}. Overall, the performance is consistent with the LK training losses that do not overfit. As we see, the method with Condition 2 only shows similar performance to the one without using any condition, and they are hard to converge in all PE evaluations. The method with Condition 1 only show significant improvements over the one with Condition 2 in all the evaluations, and outperforms DeepLK in smaller PE evaluations but becomes worse in larger PE evaluations. Using both conditions our PRISE further boosts the performance, and achieves the best among all the competitors, thanks to the learning within the overlap of solution spaces defined by both strongly star-convex constraints.

\subsubsection{Effects of Hyperparameters}


Fig. \ref{fig:PE-param} illustrates the comparisons among different hyperparameter settings using PRISE. We only evaluate one hyperparameter once while fixing the others as the default values. Overall, our approach is very robust to different hyperparameters. Specifically, Fig.~\ref{fig:PE-param} (a) demonstrates the necessity of strong star-convexity, which verifies our insight on learning a unique minimum within a local region rather than several minima (\ie not strong). When $\mu$ is sufficiently large, the performance gaps are small. Fig.~\ref{fig:PE-param} (b) shows that the midpoint choice (\ie $\lambda=0.5$) indeed provides good performance, which intuitively follows the classic results of midpoint convex and continuous functions being convex \cite{phu1997six}. Fig.~\ref{fig:PE-param} (c) shows that when $\rho$ is small, the hinge loss cannot stop the overfitting, and when it increases to a sufficiently large number, the performance will be improved significantly and stably. All the values for $\rho$ make good balance between the LK loss and the hinge loss.

\begin{table*}[h]
\centering\small
\caption{Performance comparison on MSCOCO, GoogleEarth, and GoogleMap.}
\setlength{\tabcolsep}{1pt}{
\begin{tabular}{cc|ccccccc}
\toprule
Dataset&Method& PE<0.1& PE<0.5& PE<1& PE<3& PE<5& PE<10 & PE<20\\
\hline
\multirow{6}{*}{\rotatebox[origin=c]{90}{MSCOCO}} & SIFT+RANSAC&0.00&4.70&68.32& 84.21& 90.32 &95.26&96.55 \\
& SIFT+MAGSAC \cite{barath2019magsac} &0.00& 3.66&76.27& 93.26& 94.22& 95.32& 97.26\\
& LF-Net \cite{ono2018lf} &5.60&8.62&14.20& 23.00& 78.88 &90.18&95.45 \\
& LocalTrans \cite{shao2021localtrans} &38.24&\textbf{87.25}&96.45& 98.00& 98.72 &99.25&\textbf{100.00}\\
&DHM\cite{detone2016deep}&0.00&0.00&0.87&3.48&15.27&98.22&99.96\\
&MHN\cite{le2020deep}&0.00&4.58&81.99&95.67&96.02&98.45&98.70\\
&CLKN\cite{chang2017clkn}&35.24& 83.25&83.27& 94.26& 95.75& 97.52& 98.46\\
&DeepLK\cite{zhao2021deep}&17.16& 72.25&92.81& 96.76& 97.67& 98.92&99.03\\
&\textbf{PRISE}& \textbf{52.77} $\pm$ 12.45&
83.27 $\pm$ 5.21&
\textbf{97.29} $\pm$ 1.82&\textbf{98.44} $\pm$ 1.06&\textbf{98.76} $\pm$ 0.08& \textbf{99.31} $\pm$ 0.53&
99.33 $\pm$ 1.84\\
\hline
\multirow{6}{*}{\rotatebox[origin=c]{90}{GoogleEarth}}&SIFT+RANSAC&0.18&3.42&8.97&23.09 & 41.32&50.36& 59.88 \\
&SIFT+MAGSAC \cite{barath2019magsac}&0.00& 0.00&1.88& 2.70& 3.25& 10.03& 45.29\\
&DHM\cite{detone2016deep} & 0.00 & 0.02 & 1.46 & 2.65 & 5.57 & 25.54 & 90.32 \\
&MHN\cite{le2020deep} & 0.00 & 3.42 & 4.56 &5.02 & 8.99 & 59.90 & 93.77 \\
&CLKN\cite{chang2017clkn}& \textbf{0.27} &2.88&3.45&4.24 & 4.32 & 8.77& 75.00 \\
&DeepLK\cite{zhao2021deep}& 0.00 & 3.50& 12.01& 70.20& 84.45& 90.57& 95.52\\
&\textbf{PRISE}& 0.24 $\pm$ 1.83 & \textbf{25.44} $\pm$ 1.21& \textbf{53.00} $\pm$ 1.54& \textbf{82.69} $\pm$ 1.07& \textbf{87.16} $\pm$ 1.09& \textbf{90.69} $\pm$ 0.73& \textbf{96.70} $\pm$ 0.54\\
\hline
\multirow{6}{*}{\rotatebox[origin=c]{90}{GoogleMap}}&SIFT+RANSAC&0.00&0.00&0.00&0.00 & 0.00 &2.74& 3.44 \\
&SIFT+MAGSAC \cite{barath2019magsac}&0.00& 0.00&0.00& 0.00& 0.00& 0.15& 2.58\\
&DHM\cite{detone2016deep} & 0.00 & 0.00 & 0.00 & 1.20 & 3.43 & 6.99 & 78.89 \\
&MHN\cite{le2020deep} & 0.00 & 0.34 & 0.45 &0.50 & 3.50 & 35.69 & 93.77 \\
&CLKN\cite{chang2017clkn}&0.00&0.00&0.00&1.57 & 1.88 &8.67& 22.45 \\
&DeepLK\cite{zhao2021deep}& 0.00& 2.25&16.80& 61.33& 73.39& 83.20& 93.80\\
&\textbf{PRISE}& \textbf{17.47} $\pm$ 2.44
&\textbf{48.13}$ \pm$ 12.00& \textbf{56.93} $\pm$ 3.45&\textbf{76.21} $\pm$ 2.43& \textbf{80.04} $\pm$ 5.55& \textbf{86.13} $\pm$ 0.47& \textbf{94.02} $\pm$ 1.66\\
\bottomrule
\end{tabular}}
\label{table:COCO}
\end{table*}

\subsubsection{Visualization}
Fig. \ref{fig:loss-landscape} visualizes our loss landscape comparison results based on the outputs from Stage 3, where two random entries in the ground-truth homography matrix of the selected image pair are manipulated while the other entries are fixed. 
\begin{itemize}[nosep, leftmargin=*]
    \item PRISE can generate a locally star-convex shape around the ground truth, and accordingly return an estimation that is much closer to the ground truth than the other two, which follows what we expect.

    \item All the three methods cannot generate smooth shapes over relatively large areas in the parameter space.

    \item The LK loss alone can achieve lower values, but this does not help estimation, while DeepLK and PRISE can achieve similar values overall.

\end{itemize}

Fig. \ref{fig:perf-condition} illustrates the differences in feature maps generated by DeepLK and our PRISE, where PRISE learns the feature maps more accurately, thus leading to better performance.

\subsection{State-of-the-art Comparison}
We list our comparison results in Table \ref{table:COCO}. Clearly, our PRISE significantly improves DeepLK in all the cases. Overall, PRISE achieves the best among all the competitors with dramatically performance gaps, especially when PE is small. Recall that GoogleMap is specifically designed for multimodel image alignment, and the superior performance of PRISE better demonstrates its usage in the application. We also report the standard deviation (std) for PRISE to show that our improvements are statistically significant, and very often the std is marginal to the mean.

\section{Conclusion}
Motivated by a recent work DeepLK, in this paper we propose a novel approach for multimodel image alignment, namely, Deep Star-Convexified Lucas Kanade (PRISE), to find near-optimal solutions. Our idea is to reparametrize the loss landscapes of the LK method to be star-convex using deep learning. To this end, we introduce extra hinge losses based on the definition of strong star-convexity, and impose them on the original LK loss to enforce learning star-convex loss landscapes (approximately). This leads to a minimax problem that is solvable using adversarial training. Further, to leverage the computational cost, we propose an efficient sampling based algorithm to train PRISE. We also provide some analysis on the homography estimation results from PRISE. We finally demonstrate our approach on three benchmark datasets for image alignment and show the state-of-the-art results, especially when the PE is small.

We are aware of several very recent works that report the performance on these three benchmarks, \eg Iterative Homography Network (IHN) \cite{Cao_2022_CVPR}. Unfortunately so far we cannot reproduce the results in the paper using the public code. We will try to add such new results when they are ready. Note that, however, our approach provides a general learning framework that can fit to not only the DeepLK network but also other existing networks such as IHN. In the future we will adapt our learning framework to train other networks with strongly star-convex constraints.

\section*{Acknowledgement}
Y. Zhang, X. Huang, and Z. Zhang were all supported partially by NSF CCF-2006738.

{\small
\bibliographystyle{ieee_fullname}
\bibliography{egbib}

\begin{thebibliography}{10}\itemsep=-1pt

\bibitem{afham2022crosspoint}
Mohamed Afham, Isuru Dissanayake, Dinithi Dissanayake, Amaya Dharmasiri,
  Kanchana Thilakarathna, and Ranga Rodrigo.
\newblock Crosspoint: Self-supervised cross-modal contrastive learning for 3d
  point cloud understanding.
\newblock In {\em Proceedings of the IEEE/CVF Conference on Computer Vision and
  Pattern Recognition}, pages 9902--9912, 2022.

\bibitem{agarwal2005survey}
Anubhav Agarwal, CV Jawahar, and PJ Narayanan.
\newblock A survey of planar homography estimation techniques.
\newblock {\em Centre for Visual Information Technology, Tech. Rep.
  IIIT/TR/2005/12}, 2005.

\bibitem{andriushchenko2020understanding}
Maksym Andriushchenko and Nicolas Flammarion.
\newblock Understanding and improving fast adversarial training.
\newblock {\em arXiv preprint arXiv:2007.02617}, 2020.

\bibitem{arora2019theoretical}
Sanjeev Arora, Hrishikesh Khandeparkar, Mikhail Khodak, Orestis Plevrakis, and
  Nikunj Saunshi.
\newblock A theoretical analysis of contrastive unsupervised representation
  learning.
\newblock {\em arXiv preprint arXiv:1902.09229}, 2019.

\bibitem{bachman2019learning}
Philip Bachman, R~Devon Hjelm, and William Buchwalter.
\newblock Learning representations by maximizing mutual information across
  views.
\newblock {\em Advances in neural information processing systems}, 32, 2019.

\bibitem{bai2021recent}
Tao Bai, Jinqi Luo, Jun Zhao, Bihan Wen, and Qian Wang.
\newblock Recent advances in adversarial training for adversarial robustness.
\newblock {\em arXiv preprint arXiv:2102.01356}, 2021.

\bibitem{barath2019magsac}
Daniel Barath, Jiri Matas, and Jana Noskova.
\newblock {MAGSAC}: marginalizing sample consensus.
\newblock In {\em Conference on Computer Vision and Pattern Recognition}, 2019.

\bibitem{bhat2021deep}
Goutam Bhat, Martin Danelljan, Fisher Yu, Luc Van~Gool, and Radu Timofte.
\newblock Deep reparametrization of multi-frame super-resolution and denoising.
\newblock In {\em Proceedings of the IEEE/CVF International Conference on
  Computer Vision}, pages 2460--2470, 2021.

\bibitem{Cao_2022_CVPR}
Si-Yuan Cao, Jianxin Hu, Zehua Sheng, and Hui-Liang Shen.
\newblock Iterative deep homography estimation.
\newblock In {\em Proceedings of the IEEE/CVF Conference on Computer Vision and
  Pattern Recognition (CVPR)}, pages 1879--1888, June 2022.

\bibitem{caron2020unsupervised}
Mathilde Caron, Ishan Misra, Julien Mairal, Priya Goyal, Piotr Bojanowski, and
  Armand Joulin.
\newblock Unsupervised learning of visual features by contrasting cluster
  assignments.
\newblock {\em Advances in Neural Information Processing Systems},
  33:9912--9924, 2020.

\bibitem{celledoni2022deep}
Elena Celledoni, Helge Gl{\"o}ckner, J{\o}rgen Riseth, and Alexander Schmeding.
\newblock Deep learning of diffeomorphisms for optimal reparametrizations of
  shapes.
\newblock {\em arXiv preprint arXiv:2207.11141}, 2022.

\bibitem{chang2017clkn}
Che-Han Chang, Chun-Nan Chou, and Edward~Y Chang.
\newblock Clkn: Cascaded lucas-kanade networks for image alignment.
\newblock In {\em Proceedings of the IEEE Conference on Computer Vision and
  Pattern Recognition}, pages 2213--2221, 2017.

\bibitem{chen2020simple}
Ting Chen, Simon Kornblith, Mohammad Norouzi, and Geoffrey Hinton.
\newblock A simple framework for contrastive learning of visual
  representations.
\newblock In {\em International conference on machine learning}, pages
  1597--1607. PMLR, 2020.

\bibitem{chen2020big}
Ting Chen, Simon Kornblith, Kevin Swersky, Mohammad Norouzi, and Geoffrey~E
  Hinton.
\newblock Big self-supervised models are strong semi-supervised learners.
\newblock {\em Advances in neural information processing systems},
  33:22243--22255, 2020.

\bibitem{chen2021intriguing}
Ting Chen, Calvin Luo, and Lala Li.
\newblock Intriguing properties of contrastive losses.
\newblock {\em Advances in Neural Information Processing Systems}, 34, 2021.

\bibitem{chen2020improved}
Xinlei Chen, Haoqi Fan, Ross Girshick, and Kaiming He.
\newblock Improved baselines with momentum contrastive learning.
\newblock {\em arXiv preprint arXiv:2003.04297}, 2020.

\bibitem{chopra2005learning}
Sumit Chopra, Raia Hadsell, and Yann LeCun.
\newblock Learning a similarity metric discriminatively, with application to
  face verification.
\newblock In {\em 2005 IEEE Computer Society Conference on Computer Vision and
  Pattern Recognition (CVPR'05)}, volume~1, pages 539--546. IEEE, 2005.

\bibitem{creswell2018generative}
Antonia Creswell, Tom White, Vincent Dumoulin, Kai Arulkumaran, Biswa Sengupta,
  and Anil~A Bharath.
\newblock Generative adversarial networks: An overview.
\newblock {\em IEEE signal processing magazine}, 35(1):53--65, 2018.

\bibitem{detone2016deep}
Daniel DeTone, Tomasz Malisiewicz, and Andrew Rabinovich.
\newblock Deep image homography estimation.
\newblock {\em arXiv preprint arXiv:1606.03798}, 2016.

\bibitem{dong2018boosting}
Yinpeng Dong, Fangzhou Liao, Tianyu Pang, Hang Su, Jun Zhu, Xiaolin Hu, and
  Jianguo Li.
\newblock Boosting adversarial attacks with momentum.
\newblock In {\em Proceedings of the IEEE conference on computer vision and
  pattern recognition}, pages 9185--9193, 2018.

\bibitem{du2021self}
Bi'an Du, Xiang Gao, Wei Hu, and Xin Li.
\newblock Self-contrastive learning with hard negative sampling for
  self-supervised point cloud learning.
\newblock In {\em Proceedings of the 29th ACM International Conference on
  Multimedia}, pages 3133--3142, 2021.

\bibitem{eckart2021self}
Benjamin Eckart, Wentao Yuan, Chao Liu, and Jan Kautz.
\newblock Self-supervised learning on 3d point clouds by learning discrete
  generative models.
\newblock In {\em Proceedings of the IEEE/CVF Conference on Computer Vision and
  Pattern Recognition}, pages 8248--8257, 2021.

\bibitem{erlik2017homography}
Farzan Erlik~Nowruzi, Robert Laganiere, and Nathalie Japkowicz.
\newblock Homography estimation from image pairs with hierarchical
  convolutional networks.
\newblock In {\em Proceedings of the IEEE International Conference on Computer
  Vision Workshops}, pages 913--920, 2017.

\bibitem{fu2019adaptive}
Jun Fu, Jing Liu, Yuhang Wang, Yong Li, Yongjun Bao, Jinhui Tang, and Hanqing
  Lu.
\newblock Adaptive context network for scene parsing.
\newblock In {\em Proceedings of the IEEE/CVF International Conference on
  Computer Vision}, pages 6748--6757, 2019.

\bibitem{goodfellow2020generative}
Ian Goodfellow, Jean Pouget-Abadie, Mehdi Mirza, Bing Xu, David Warde-Farley,
  Sherjil Ozair, Aaron Courville, and Yoshua Bengio.
\newblock Generative adversarial networks.
\newblock {\em Communications of the ACM}, 63(11):139--144, 2020.

\bibitem{goodfellow2014explaining}
Ian~J Goodfellow, Jonathon Shlens, and Christian Szegedy.
\newblock Explaining and harnessing adversarial examples.
\newblock {\em arXiv preprint arXiv:1412.6572}, 2014.

\bibitem{gower2021sgd}
Robert Gower, Othmane Sebbouh, and Nicolas Loizou.
\newblock Sgd for structured nonconvex functions: Learning rates, minibatching
  and interpolation.
\newblock In {\em International Conference on Artificial Intelligence and
  Statistics}, pages 1315--1323. PMLR, 2021.

\bibitem{hadsell2006dimensionality}
Raia Hadsell, Sumit Chopra, and Yann LeCun.
\newblock Dimensionality reduction by learning an invariant mapping.
\newblock In {\em 2006 IEEE Computer Society Conference on Computer Vision and
  Pattern Recognition (CVPR'06)}, volume~2, pages 1735--1742. IEEE, 2006.

\bibitem{he2020momentum}
Kaiming He, Haoqi Fan, Yuxin Wu, Saining Xie, and Ross Girshick.
\newblock Momentum contrast for unsupervised visual representation learning.
\newblock In {\em Proceedings of the IEEE/CVF conference on computer vision and
  pattern recognition}, pages 9729--9738, 2020.

\bibitem{pmlr-v125-hinder20a}
Oliver Hinder, Aaron Sidford, and Nimit Sohoni.
\newblock Near-optimal methods for minimizing star-convex functions and beyond.
\newblock In {\em Proceedings of Thirty Third Conference on Learning Theory},
  pages 1894--1938, 2020.

\bibitem{hjelm2018learning}
R~Devon Hjelm, Alex Fedorov, Samuel Lavoie-Marchildon, Karan Grewal, Phil
  Bachman, Adam Trischler, and Yoshua Bengio.
\newblock Learning deep representations by mutual information estimation and
  maximization.
\newblock {\em arXiv preprint arXiv:1808.06670}, 2018.

\bibitem{jaiswal2020survey}
Ashish Jaiswal, Ashwin~Ramesh Babu, Mohammad~Zaki Zadeh, Debapriya Banerjee,
  and Fillia Makedon.
\newblock A survey on contrastive self-supervised learning.
\newblock {\em Technologies}, 9(1):2, 2020.

\bibitem{jiang2021guided}
Li Jiang, Shaoshuai Shi, Zhuotao Tian, Xin Lai, Shu Liu, Chi-Wing Fu, and Jiaya
  Jia.
\newblock Guided point contrastive learning for semi-supervised point cloud
  semantic segmentation.
\newblock In {\em Proceedings of the IEEE/CVF International Conference on
  Computer Vision}, pages 6423--6432, 2021.

\bibitem{kingma2014adam}
Diederik~P Kingma and Jimmy Ba.
\newblock Adam: A method for stochastic optimization.
\newblock {\em arXiv preprint arXiv:1412.6980}, 2014.

\bibitem{kuruzov2021sequential}
Ilya~A Kuruzov and Fedor~S Stonyakin.
\newblock Sequential subspace optimization for quasar-convex optimization
  problems with inexact gradient.
\newblock In {\em International Conference on Optimization and Applications},
  pages 19--33. Springer, 2021.

\bibitem{le2020deep}
Hoang Le, Feng Liu, Shu Zhang, and Aseem Agarwala.
\newblock Deep homography estimation for dynamic scenes.
\newblock In {\em Proceedings of the IEEE/CVF Conference on Computer Vision and
  Pattern Recognition}, pages 7652--7661, 2020.

\bibitem{le2020contrastive}
Phuc~H Le-Khac, Graham Healy, and Alan~F Smeaton.
\newblock Contrastive representation learning: A framework and review.
\newblock {\em IEEE Access}, 8:193907--193934, 2020.

\bibitem{lee2016optimizing}
Jasper~CH Lee and Paul Valiant.
\newblock Optimizing star-convex functions.
\newblock In {\em 2016 IEEE 57th Annual Symposium on Foundations of Computer
  Science (FOCS)}, pages 603--614. IEEE, 2016.

\bibitem{li2018visualizing}
Hao Li, Zheng Xu, Gavin Taylor, Christoph Studer, and Tom Goldstein.
\newblock Visualizing the loss landscape of neural nets.
\newblock {\em Advances in neural information processing systems}, 31, 2018.

\bibitem{li2020prototypical}
Junnan Li, Pan Zhou, Caiming Xiong, and Steven~CH Hoi.
\newblock Prototypical contrastive learning of unsupervised representations.
\newblock {\em arXiv preprint arXiv:2005.04966}, 2020.

\bibitem{li2017convergence}
Yuanzhi Li and Yang Yuan.
\newblock Convergence analysis of two-layer neural networks with relu
  activation.
\newblock {\em Advances in Neural Information Processing Systems}, 30:597--607,
  2017.

\bibitem{lin2019nesterov}
Jiadong Lin, Chuanbiao Song, Kun He, Liwei Wang, and John~E Hopcroft.
\newblock Nesterov accelerated gradient and scale invariance for adversarial
  attacks.
\newblock {\em arXiv preprint arXiv:1908.06281}, 2019.

\bibitem{lin2014microsoft}
Tsung-Yi Lin, Michael Maire, Serge Belongie, James Hays, Pietro Perona, Deva
  Ramanan, Piotr Doll{\'a}r, and C~Lawrence Zitnick.
\newblock Microsoft coco: Common objects in context.
\newblock In {\em European conference on computer vision}, pages 740--755.
  Springer, 2014.

\bibitem{loh2013regularized}
Po-Ling Loh and Martin~J Wainwright.
\newblock Regularized m-estimators with nonconvexity: Statistical and
  algorithmic theory for local optima.
\newblock {\em arXiv preprint arXiv:1305.2436}, 2013.

\bibitem{loh2015regularized}
Po-Ling Loh and Martin~J Wainwright.
\newblock Regularized m-estimators with nonconvexity: Statistical and
  algorithmic theory for local optima.
\newblock {\em The Journal of Machine Learning Research}, 16(1):559--616, 2015.

\bibitem{lucasiterative}
Bruce~D Lucas and Takeo Kanade.
\newblock An iterative image registration technique with an application to
  stereo vision.
\newblock In {\em Proceedings of Imaging Understanding Workshop}, pages
  121--130, 1981.

\bibitem{mao2016successive}
Yuanqi Mao, Michael Szmuk, and Beh{\c{c}}et A{\c{c}}{\i}kme{\c{s}}e.
\newblock Successive convexification of non-convex optimal control problems and
  its convergence properties.
\newblock In {\em 2016 IEEE 55th Conference on Decision and Control (CDC)},
  pages 3636--3641. IEEE, 2016.

\bibitem{misra2020self}
Ishan Misra and Laurens van~der Maaten.
\newblock Self-supervised learning of pretext-invariant representations.
\newblock In {\em Proceedings of the IEEE/CVF Conference on Computer Vision and
  Pattern Recognition}, pages 6707--6717, 2020.

\bibitem{nesterov2006cubic}
Yurii Nesterov and Boris~T Polyak.
\newblock Cubic regularization of newton method and its global performance.
\newblock {\em Mathematical Programming}, 108(1):177--205, 2006.

\bibitem{nguyen2018unsupervised}
Ty Nguyen, Steven~W Chen, Shreyas~S Shivakumar, Camillo~Jose Taylor, and Vijay
  Kumar.
\newblock Unsupervised deep homography: A fast and robust homography estimation
  model.
\newblock {\em IEEE Robotics and Automation Letters}, 3(3):2346--2353, 2018.

\bibitem{ono2018lf}
Yuki Ono, Eduard Trulls, Pascal Fua, and Kwang~Moo Yi.
\newblock Lf-net: Learning local features from images.
\newblock {\em Advances in neural information processing systems}, 31, 2018.

\bibitem{oord2018representation}
Aaron van~den Oord, Yazhe Li, and Oriol Vinyals.
\newblock Representation learning with contrastive predictive coding.
\newblock {\em arXiv preprint arXiv:1807.03748}, 2018.

\bibitem{phu1997six}
HX Phu.
\newblock Six kinds of roughly convex functions.
\newblock {\em Journal of optimization theory and applications},
  92(2):357--375, 1997.

\bibitem{qian2022survey}
Zhuang Qian, Kaizhu Huang, Qiu-Feng Wang, and Xu-Yao Zhang.
\newblock A survey of robust adversarial training in pattern recognition:
  Fundamental, theory, and methodologies.
\newblock {\em arXiv preprint arXiv:2203.14046}, 2022.

\bibitem{reddi2018adaptive}
S Reddi, Manzil Zaheer, Devendra Sachan, Satyen Kale, and Sanjiv Kumar.
\newblock Adaptive methods for nonconvex optimization.
\newblock In {\em Proceeding of 32nd Conference on Neural Information
  Processing Systems (NIPS 2018)}, 2018.

\bibitem{shafahi2019adversarial}
Ali Shafahi, Mahyar Najibi, Amin Ghiasi, Zheng Xu, John Dickerson, Christoph
  Studer, Larry~S Davis, Gavin Taylor, and Tom Goldstein.
\newblock Adversarial training for free!
\newblock {\em arXiv preprint arXiv:1904.12843}, 2019.

\bibitem{shao2022scrnet}
Huixiang Shao, Zhijiang Zhang, Xiaoyu Feng, and Dan Zeng.
\newblock Scrnet: A spatial consistency guided network using contrastive
  learning for point cloud registration.
\newblock {\em Symmetry}, 14(1):140, 2022.

\bibitem{shao2021localtrans}
Ruizhi Shao, Gaochang Wu, Yuemei Zhou, Ying Fu, Lu Fang, and Yebin Liu.
\newblock Localtrans: A multiscale local transformer network for
  cross-resolution homography estimation.
\newblock In {\em IEEE Conference on Computer Vision (ICCV 2021)}, 2021.

\bibitem{sohn2016improved}
Kihyuk Sohn.
\newblock Improved deep metric learning with multi-class n-pair loss objective.
\newblock {\em Advances in neural information processing systems}, 29, 2016.

\bibitem{tang2022contrastive}
Liyao Tang, Yibing Zhan, Zhe Chen, Baosheng Yu, and Dacheng Tao.
\newblock Contrastive boundary learning for point cloud segmentation.
\newblock In {\em Proceedings of the IEEE/CVF Conference on Computer Vision and
  Pattern Recognition}, pages 8489--8499, 2022.

\bibitem{tian2020contrastive}
Yonglong Tian, Dilip Krishnan, and Phillip Isola.
\newblock Contrastive multiview coding.
\newblock In {\em European conference on computer vision}, pages 776--794.
  Springer, 2020.

\bibitem{tian2020makes}
Yonglong Tian, Chen Sun, Ben Poole, Dilip Krishnan, Cordelia Schmid, and
  Phillip Isola.
\newblock What makes for good views for contrastive learning?
\newblock {\em Advances in Neural Information Processing Systems},
  33:6827--6839, 2020.

\bibitem{tuan2000low}
Hoang~Duong Tuan and Pierre Apkarian.
\newblock Low nonconvexity-rank bilinear matrix inequalities: algorithms and
  applications in robust controller and structure designs.
\newblock {\em IEEE Transactions on Automatic Control}, 45(11):2111--2117,
  2000.

\bibitem{vettam2019regularized}
Sujit Vettam and Majnu John.
\newblock Regularized deep learning with nonconvex penalties.
\newblock {\em arXiv preprint arXiv:1909.05142}, 2019.

\bibitem{wang2022improving}
Di Wang, Lulu Tang, Xu Wang, Luqing Luo, and Zhi-Xin Yang.
\newblock Improving deep learning on point cloud by maximizing mutual
  information across layers.
\newblock {\em Pattern Recognition}, 131:108892, 2022.

\bibitem{wang2021understanding}
Feng Wang and Huaping Liu.
\newblock Understanding the behaviour of contrastive loss.
\newblock In {\em Proceedings of the IEEE/CVF conference on computer vision and
  pattern recognition}, pages 2495--2504, 2021.

\bibitem{wang2020adaptively}
Huachuan Wang and James Ting-Ho Lo.
\newblock Adaptively solving the local-minimum problem for deep neural
  networks.
\newblock {\em arXiv preprint arXiv:2012.13632}, 2020.

\bibitem{wang2020understanding}
Tongzhou Wang and Phillip Isola.
\newblock Understanding contrastive representation learning through alignment
  and uniformity on the hypersphere.
\newblock In {\em International Conference on Machine Learning}, pages
  9929--9939. PMLR, 2020.

\bibitem{wong2020fast}
Eric Wong, Leslie Rice, and J~Zico Kolter.
\newblock Fast is better than free: Revisiting adversarial training.
\newblock {\em arXiv preprint arXiv:2001.03994}, 2020.

\bibitem{wu2018unsupervised}
Zhirong Wu, Yuanjun Xiong, Stella~X Yu, and Dahua Lin.
\newblock Unsupervised feature learning via non-parametric instance
  discrimination.
\newblock In {\em Proceedings of the IEEE conference on computer vision and
  pattern recognition}, pages 3733--3742, 2018.

\bibitem{yan2022implicit}
Siming Yan, Zhenpei Yang, Haoxiang Li, Li Guan, Hao Kang, Gang Hua, and Qixing
  Huang.
\newblock Implicit autoencoder for point cloud self-supervised representation
  learning.
\newblock {\em arXiv preprint arXiv:2201.00785}, 2022.

\bibitem{yang2021unsupervised}
Cheng-Kun Yang, Yung-Yu Chuang, and Yen-Yu Lin.
\newblock Unsupervised point cloud object co-segmentation by co-contrastive
  learning and mutual attention sampling.
\newblock In {\em Proceedings of the IEEE/CVF International Conference on
  Computer Vision}, pages 7335--7344, 2021.

\bibitem{yang2020graduated}
Heng Yang, Pasquale Antonante, Vasileios Tzoumas, and Luca Carlone.
\newblock Graduated non-convexity for robust spatial perception: From
  non-minimal solvers to global outlier rejection.
\newblock {\em IEEE Robotics and Automation Letters}, 5(2):1127--1134, 2020.

\bibitem{ye2019fast}
Yuanxin Ye, Lorenzo Bruzzone, Jie Shan, Francesca Bovolo, and Qing Zhu.
\newblock Fast and robust matching for multimodal remote sensing image
  registration.
\newblock {\em IEEE Transactions on Geoscience and Remote Sensing},
  57(11):9059--9070, 2019.

\bibitem{ye2017robust}
Yuanxin Ye, Jie Shan, Lorenzo Bruzzone, and Li Shen.
\newblock Robust registration of multimodal remote sensing images based on
  structural similarity.
\newblock {\em IEEE Transactions on Geoscience and Remote Sensing},
  55(5):2941--2958, 2017.

\bibitem{zhang2020content}
Jirong Zhang, Chuan Wang, Shuaicheng Liu, Lanpeng Jia, Nianjin Ye, Jue Wang, Ji
  Zhou, and Jian Sun.
\newblock Content-aware unsupervised deep homography estimation.
\newblock In {\em European Conference on Computer Vision}, pages 653--669.
  Springer, 2020.

\bibitem{zhao2021deep}
Yiming Zhao, Xinming Huang, and Ziming Zhang.
\newblock Deep lucas-kanade homography for multimodal image alignment.
\newblock In {\em Proceedings of the IEEE/CVF Conference on Computer Vision and
  Pattern Recognition}, pages 15950--15959, 2021.

\bibitem{zhou2018sgd}
Yi Zhou, Junjie Yang, Huishuai Zhang, Yingbin Liang, and Vahid Tarokh.
\newblock {SGD} converges to global minimum in deep learning via star-convex
  path.
\newblock In {\em International Conference on Learning Representations}, 2019.

\bibitem{zimmermann2021contrastive}
Roland~S Zimmermann, Yash Sharma, Steffen Schneider, Matthias Bethge, and
  Wieland Brendel.
\newblock Contrastive learning inverts the data generating process.
\newblock In {\em International Conference on Machine Learning}, pages
  12979--12990. PMLR, 2021.

\end{thebibliography}
}
\end{document}